%% file: paper.tex
\documentclass{teleai}
\usepackage{natbib,amssymb}
\usepackage[T1]{fontenc}
\usepackage{adjustbox}
\usepackage{amsmath,amsfonts,bm}

\def\eqref#1{equation~\ref{#1}}
\def\1{\bm{1}}

\DeclareMathAlphabet{\mathsfit}{\encodingdefault}{\sfdefault}{m}{sl}
\SetMathAlphabet{\mathsfit}{bold}{\encodingdefault}{\sfdefault}{bx}{n}

\usepackage{graphicx,booktabs,tabularx}
\usepackage{amsthm}
\usepackage[nopatch=eqnum]{microtype}
\usepackage{enumitem}
\usepackage{float}
\usepackage{xurl}
\usepackage{marvosym}

\floatstyle{ruled}
\newfloat{algorithm}{tbp}{loa}
\floatname{algorithm}{Algorithm}
\newtheorem{proposition}{Proposition}
\newcommand{\method}{UniWorld}
\newcommand{\concat}{\operatorname{concat}}
\newcommand{\readdepth}{\operatorname{ReadDepth}}
\newcommand{\proj}{\mathcal{P}}
\newcommand{\miB}{\,\mathrm{MiB}}

\hypersetup{pdfauthor={Xin Zhang, Yabo Chen, Shuaiwei Wang, Zixuan Duan, Haibin Huang, Feng Xu, Xuelong Li},pdftitle={UniWorld: A Unified World Model for Real-Time Visual World Simulation}}
\title{UniWorld: A Unified World Model for Real-Time Visual World Simulation}

\newcommand{\equalcontribmark}{\textsuperscript{*}}
\newcommand{\correspondingmark}{\textsuperscript{\Letter}}
\author[1,2]{Xin Zhang\equalcontribmark\,}
\author[2]{Yabo Chen\equalcontribmark\,}
\author[2,3]{Shuaiwei Wang}
\author[2,4]{Zixuan Duan}
\author[2]{Haibin Huang}
\author[1]{Feng~Xu}
\author[2]{Xuelong Li\correspondingmark\,}

\affiliation[1]{Fudan University}
\affiliation[2]{Institute of Artificial Intelligence (TeleAI), China Telecom}
\affiliation[3]{Jilin University}
\affiliation[4]{Nanjing University}
\contribution{\textsuperscript{*}Equal contributions. \correspondingmark Corresponding authors.}
\metadata[Keywords]{\parbox[t]{\dimexpr\linewidth-4.2cm\relax}{World models; Video generation; Depth conditioning;\\Continuous visual simulation}}

\abstract{Continuous visual world simulation requires coupling evolving scene geometry with temporally persistent appearance. We present UniWorld, a unified framework that integrates depth trajectories and reference appearance within a shared video generation model. The central design maintains visual and geometric histories on a common timeline. Noise-aligned overlapping continuation places generated history at the prescribed sampling noise level and commits only newly generated latents. Rolling depth latent state preserves the corresponding geometric controls, incrementally reads newly available geometry, and keeps the conditioning fixed within each sampling window. Together, these mechanisms ensure exact preservation of latent history while bounding the active depth state by the window size. We evaluate continuous generation, paired history interventions, and long-horizon behavior across diverse scene categories. The same generator supports estimated, native, and on-demand simulator depth, while cached and incremental execution paths produce bitwise-identical outputs. Appearance changes and geometric control branches operate through the same interface without modifying committed history. The measured steady-state DiT execution path achieves real-time generation throughput, demonstrating the efficiency of UniWorld for continuous visual world simulation. Project page: \url{https://zephipage.github.io/UniWorld}.}

\begin{document}
\maketitle

\begin{figure}[htbp]
  \centering
  \includegraphics[width=\linewidth,pagebox=mediabox]{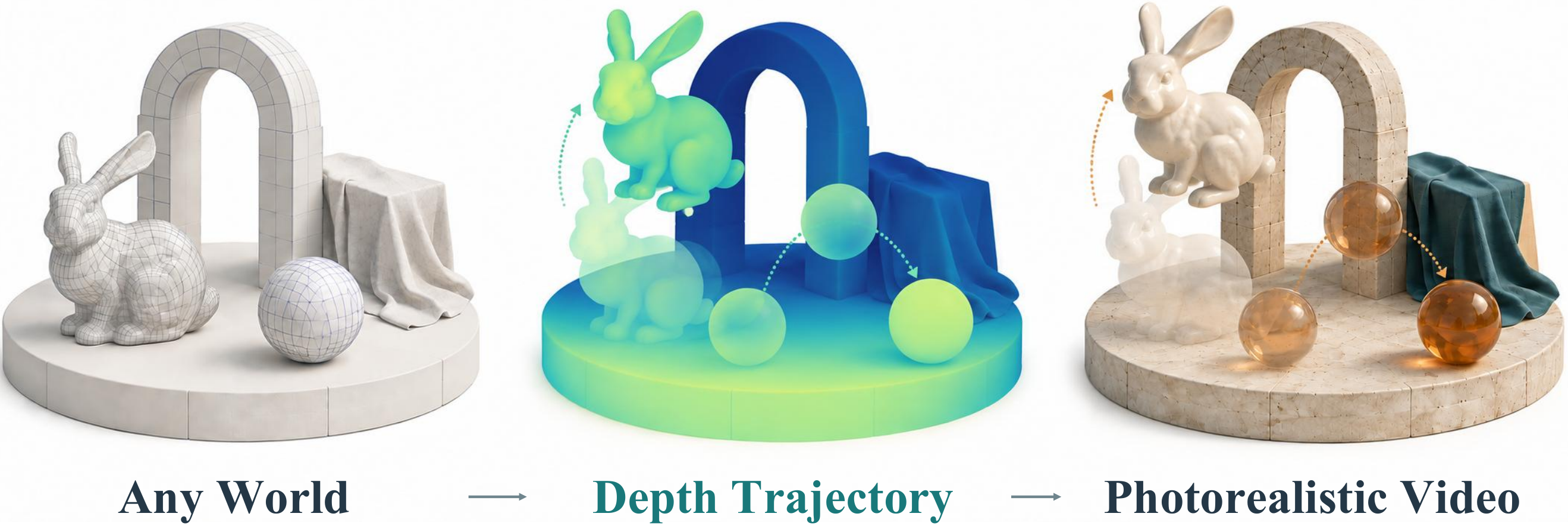}
  \caption{\textbf{UniWorld: from world geometry to continuous visual observations.}
  Conceptual overview: depth trajectories describe geometric evolution, while reference RGB specifies appearance for a shared video renderer.}
  \label{fig:teaser}
\end{figure}
\section{Introduction}
\label{sec:introduction}

World models connect observations with environment dynamics, supporting simulation, prediction, and decision-making~\citep{ha2018world,bruce2024genie,zhu2026astra}. Graphics engines and physical simulators provide structured and controllable scene states, whereas video generators provide rich visual appearance. Bridging these capabilities requires a renderer that can continuously transform evolving geometry into temporally consistent visual observations. We focus on this simulator-to-renderer component: an upstream world source specifies a geometric trajectory, and a generative model produces the corresponding visual realization.

Different world sources expose distinct state variables, camera models, and control interfaces, which makes direct coupling to a shared visual generator difficult. We observe that the visible geometry produced by these sources can instead be represented in a common image-space form: depth. Image-aligned depth captures scene layout, occlusion, and geometric changes, while a reference RGB image specifies texture, color, and appearance semantics. This separation maps heterogeneous world states to a common rendering interface, enabling the same generator to operate on recorded scenes, simulator trajectories, and on-demand geometric inputs.

Continuous rendering introduces a second challenge: generation must remain consistent across finite sampling windows. Long-video methods extend temporal context, use overlapping windows, or compress history~\citep{henschel2024streaming,zhang2025framepack,gao2025longvie}, while autoregressive approaches condition future generation on earlier observations~\citep{huang2025selfforcing,zhu2026astra}. Depth-driven rendering must additionally preserve the geometric conditions associated with those observations. Continuous generation therefore requires two synchronized histories: the generated visual state and the depth state that produced it. Previously generated observations must remain fixed, the associated depth values must remain unchanged, and only newly requested geometry should produce newly committed outputs. A shared temporal contract makes these state transitions explicit and testable.

We propose \method{}, a unified generative rendering framework that maintains visual and geometric states on a common timeline. \emph{Noise-aligned overlapping continuation} places previously generated latents in the prefix of the next window and constrains them to the corresponding flow time throughout sampling. \emph{Rolling depth latent state} preserves the matching geometric history, incrementally reads only newly available depth positions, and keeps the conditioning fixed within each window. After sampling, the system commits only the newly generated video suffix. The sampling window can therefore advance without rewriting committed latents or altering the controls associated with the overlapping region.

This design bounds the active state and per-window computation independently of the total sequence duration, thereby enabling continuous long-horizon rendering. We verify the state contract through actual model executions and evaluate \method{} across four scene categories using estimated and native depth, controlled fixed-window interventions, and extended trajectories. An on-demand engine further demonstrates reference-appearance changes and branch-conditioned continuation through the same interface. The measured steady-state DiT execution path achieves real-time rendering throughput while maintaining a bounded active state during continuous generation.

Our contributions are:
\begin{enumerate}[leftmargin=*,itemsep=2pt,topsep=3pt]
  \item \textbf{UniWorld.} A unified visual simulation framework that represents heterogeneous geometric trajectories using image-aligned depth and combines them with reference appearance, providing a common generative rendering interface for recorded scenes, simulators, and on-demand world trajectories.
  \item \textbf{Noise-aligned overlapping continuation.} A sampling rule that constrains a generated multi-frame prefix at every flow time and commits only the newly generated suffix, thereby preserving previously generated latents while retaining temporal context for continued generation.
  \item \textbf{Rolling depth latent state.} A synchronized geometric state that incrementally reads only newly available depth positions, preserves the conditions associated with committed observations, and maintains an active-state bound that is independent of sequence duration.
\end{enumerate}

\section{Related Work}
\label{sec:related}

\subsection{Visual World Models}
World Models learn compact representations of environment dynamics~\citep{ha2018world}, and Genie learns interactive visual environments from video~\citep{bruce2024genie}. Astra combines autoregressive denoising, noisy histories, and action adapters to support exploration, manipulation, and driving~\citep{zhu2026astra}. These approaches link observations to actions or latent controls. \method{} complements these approaches by providing a shared visual rendering interface: an upstream simulator supplies the evolving geometry, and the generator continuously renders the corresponding appearance.

\subsection{Conditional Video Generation}
Large video models support conditioning on text and images~\citep{wan2025wan,minimax2026h3}. ControlNet, SparseCtrl, and CameraCtrl introduce structural or camera constraints~\citep{zhang2023controlnet,guo2023sparsectrl,he2024cameractrl}. LongVie and LongVie 2 combine multiple control signals for long-video synthesis~\citep{gao2025longvie,gao2025longvie2}. Our implementation builds on H3, Video DeltaNet (VDN), Qwen3-VL, and the VideoX-Fun Union control branch~\citep{minimax2026h3,xi2026vdn,bai2025qwen3vl,aigcapps2026videox}. The original image provides persistent semantic context, while a generated boundary image and temporally aligned depth serve as local conditioning inputs. Our modifications focus on how these conditioning inputs and the generated history are retained across windows.

\subsection{Long-Video Generation}
StreamingT2V combines short- and long-term appearance memory~\citep{henschel2024streaming}, and FramePack compresses frame context~\citep{zhang2025framepack}. Diffusion Forcing uses per-token noise levels during training, and Self Forcing addresses the mismatch between training histories and model-generated histories~\citep{chen2024diffusionforcing,huang2025selfforcing}. RePaint conditions diffusion sampling on known regions by replacing noisy observations~\citep{lugmayr2022repaint}. LongVie uses unified noise initialization and global control normalization~\citep{gao2025longvie}. Building on these principles, we focus on a joint temporal contract for a flow renderer that uses eight evaluations. Under this contract, the multi-frame target prefix is matched to the current sampling noise level, canonical latent values are preserved in overlapping depth segments, and incremental output is aligned with the global VAE phase.

\section{Approach}
\label{sec:method}

\subsection{Task Definition and Preliminaries}
\label{sec:task}
The method takes as input a single reference image $I_{\mathrm{ref}}$, optional text $p$, and a spatially and temporally aligned depth sequence $D$, and outputs a video $\hat V$ at playback rate $f$. The target generation rate for real-time operation is 14 FPS. For subsequent windows, boundary images are obtained from previously generated content. All temporal intervals use zero-based, half-open indexing. Let $N$ and $F$ denote the aligned temporal lengths in pixel space and latent space, respectively. Window $i$ spans $\Omega_i=[a_i,b_i)$ in latent time, with length $L_i=b_i-a_i$, inherited history $H_i$ of length $h_i$, depth condition $C_i$, boundary image $I_i$, and a slice $\epsilon_i$ of the global noise. The first window has $h_0=0$.
\begin{equation}
 c_{\mathrm g}=E_{\mathrm{Qwen}}(I_{\mathrm{ref}},p),\qquad
 \hat Z_i=G_\theta^{(K)}(H_i,C_i,I_i,c_{\mathrm g},\epsilon_i),\qquad K=8.
 \label{eq:generator}
\end{equation}
Here $G_\theta^{(K)}$ includes the projection and sampler described below. The windows are generated sequentially because each window depends on previously generated history, whereas generation within each window remains bidirectional.

\subsection{UniWorld: A Unified Framework for Visual World Simulation}
\label{sec:framework}
\begin{figure}[t]
  \centering
  \includegraphics[width=\linewidth]{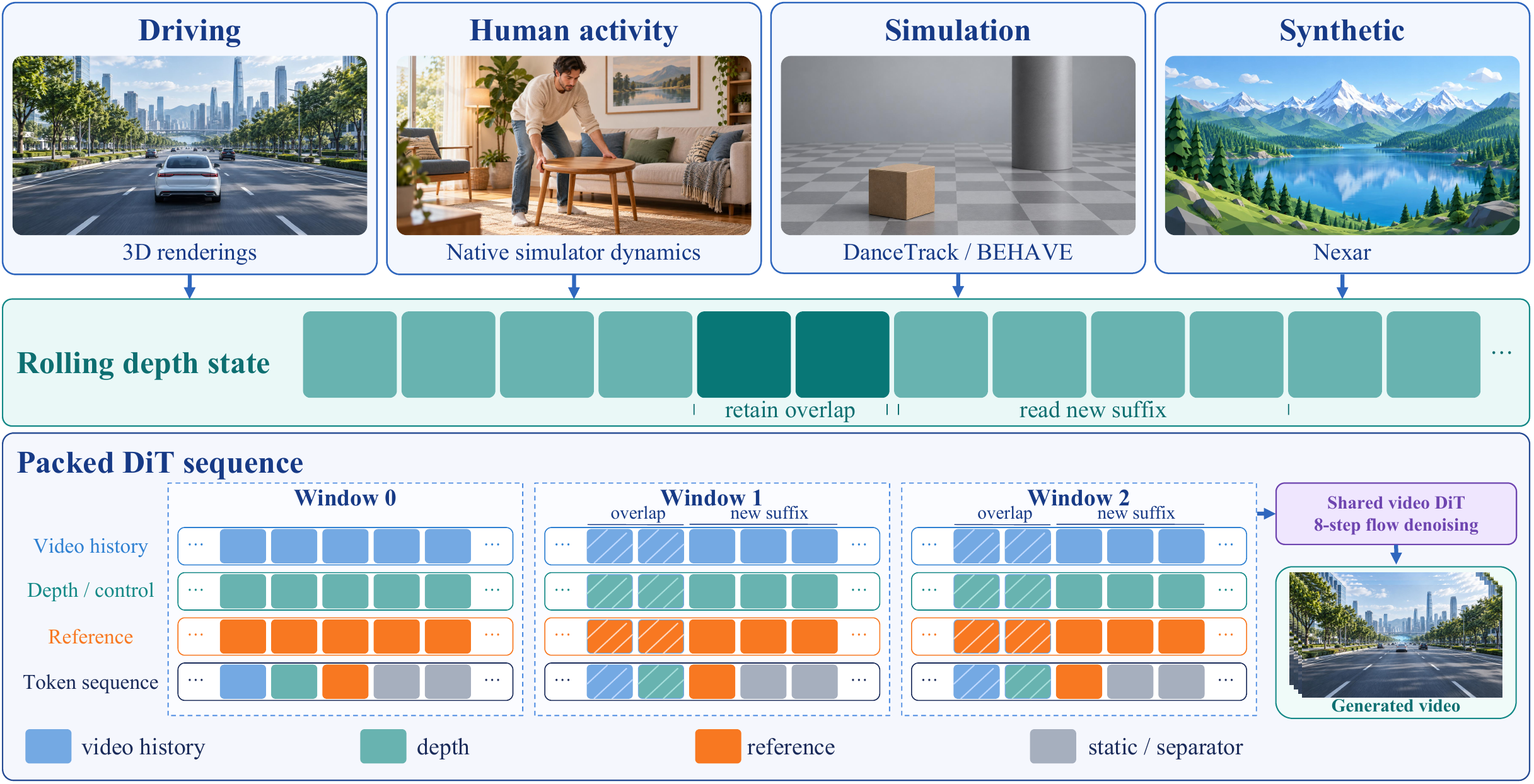}
  \caption{\textbf{UniWorld: a unified renderer for heterogeneous world sources.} Driving, human-activity, simulation, and synthetic sources are converted into a common rolling depth state. Each continuation window packs video history, depth/control, reference, and separator tokens into a shared DiT sequence. Overlapping positions are retained, only the newly required depth suffix is read, and only the newly generated video suffix is committed to the output stream.}
  \label{fig:method}
\end{figure}

For source $m$, let $\xi^{(m)}$ denote its state or observations, and let $\kappa^{(m)}$ denote the corresponding camera and alignment information. A source adapter $\mathcal A_m$ produces visible depth, while $\mathcal N_m$ standardizes its representation:
\begin{equation}
 \bar D^{(m)}=\mathcal N_m\!\left(\mathcal A_m(\xi^{(m)},\kappa^{(m)})\right),\qquad
 C=\operatorname{Patchify}\!\left[\operatorname{Norm}_z\!\left(E_{\mathrm{VAE}}(\operatorname{RGBReplicate}(\bar D))\right)\right].
 \label{eq:depth_encoding}
\end{equation}
Time, viewpoint, crop, near/far direction, scale, and encoding phase remain consistent throughout the trajectory. We replicate grayscale depth across the RGB channels, use the VAE posterior mode with FP32 latent normalization, and patchify with shape $(1,2,2)$. Each latent time position contains $S=(H_{\mathrm{px}}/32)(W_{\mathrm{px}}/32)$ spatial tokens, each with dimension 96. Both estimated and native depth follow the verified far-bright convention of the checkpoint. An online adapter encodes newly requested geometry under the same global phase; Appendix~\ref{app:depth_protocol} specifies the source mappings.

Appearance conditioning is provided through two paths (Figure~\ref{fig:method}). The original image and text are encoded once by Qwen3-VL and reused throughout the trajectory. A separate VAE first-frame condition uses $I_0=I_{\mathrm{ref}}$ initially and generated RGB at the start of each subsequent window. These independent conditioning rows are distinct from the multi-frame history, which occupies the target video sequence itself.

\subsection{Noise-Aligned Overlapping Continuation}
\label{sec:overlap}

\paragraph{History as a sampling boundary condition.}
We adopt the noise-to-data flow convention~\citep{lipman2023flow}:
\begin{equation}
 z_t=t z_{\mathrm{clean}}+(1-t)\epsilon,\qquad
 v=z_{\mathrm{clean}}-\epsilon,\qquad 0\leq t\leq1.
 \label{eq:flow}
\end{equation}
For window $i$, let $M_i$ mask the first $h_i$ latent positions corresponding to inherited history, and extend $H_i$ by zeros outside this prefix. We impose the history through the projection
\begin{equation}
 \boxed{
 \proj_{i,t}(z)
 =
 M_i\odot\bigl[tH_i+(1-t)\epsilon_i\bigr]
 +(1-M_i)\odot z.
 }
 \label{eq:projection}
\end{equation}
Thus, the inherited prefix follows the same noise-to-data trajectory as the current sample instead of being inserted as a clean latent into a noisy state. Global positions shared by overlapping windows use the same realization of global noise.

On the shifted flow grid $0=t_0<\cdots<t_K=1$, we initialize
\begin{equation}
 z_i^{(0)}=\proj_{i,t_0}(\epsilon_i),
\end{equation}
and perform
\begin{equation}
 \tilde z_i^{(k+1)}
 =
 z_i^{(k)}
 +(t_{k+1}-t_k)
 v_\theta\!\left(
 z_i^{(k)},t_k;\mathcal C_i
 \right),
 \qquad
 z_i^{(k+1)}
 =
 \proj_{i,t_{k+1}}\!\left(\tilde z_i^{(k+1)}\right),
 \label{eq:update}
\end{equation}
where
\begin{equation}
 \mathcal C_i=(C_i,I_i,c_{\mathrm g})
\end{equation}
collects the fixed depth, boundary-image, and global appearance conditions for window $i$.

The projection ensures that the inherited prefix is aligned with the current flow noise level before every model evaluation. After each numerical update, the predicted prefix is replaced by its prescribed history trajectory, whereas the predicted suffix is retained. Consequently, the inherited history fully participates in the bidirectional forward pass within each window and guides newly generated content without drifting. At $t_K=1$, the projected prefix exactly recovers $H_i$. Because the projection is an element-wise replacement, it requires no additional model evaluations.

\paragraph{History preservation and output commitment.}
For fixed history, mask, and noise, the projection is idempotent. The initialization and equation~\ref{eq:update} imply
\begin{equation}
 \proj_{i,t}(\proj_{i,t}(z))=\proj_{i,t}(z),\qquad
 M_i\odot z_i^{(k)}=M_i\odot[t_kH_i+(1-t_k)\epsilon_i].
 \label{eq:invariant}
\end{equation}
Thus, the terminal prefix equals $H_i$. We preserve the previously committed sequence $\hat Z_{<i}$ and append only the newly generated suffix:
\begin{equation}
 \hat Z_{\leq i}=\concat\bigl(\hat Z_{<i},z_i^{(K)}[h_i:L_i]\bigr).
 \label{eq:commit}
\end{equation}
Appendix~\ref{app:proofs} proves exact latent-history preservation. Section~\ref{sec:ablations} evaluates the generated suffix under fixed-window interventions, while the decoder described below preserves the global phase during RGB commitment.

\paragraph{Window geometry and boundary images.}
Under the temporal geometry of the video VAE, \(N=17n+5\) pixel frames correspond to \(F=5n+2\) latent frames. The default pixel window, overlap, and stride are \((W,O,\Delta)=(345,39,306)\), corresponding to \((L_{\max},h,\Delta_{\mathrm{latent}})=(102,12,90)\). For the evaluation stream at 24 FPS, a 30-second request is internally padded to 736 pixel frames and 217 latents, and the decoded output is trimmed to 720 frames. The corresponding latent windows are \([0,102)\), \([90,192)\), and \([180,217)\) (Figure~\ref{fig:timeline}). The subsequent windows start at pixel frames 306 and 612, while the nominal new-suffix boundaries are 345 and 651.

The next boundary RGB is taken at the \emph{start of the next window}: decoding latents \([85,97)\) and selecting local pixel frame 17 recovers frame 306, matching full-window decoding. RGB export preserves a single global VAE phase. Incremental delivery retains five unfinished blend frames and emits \([0,340)\), \([340,646)\), and then \([646,720)\); these pixel-delivery boundaries differ from the latent commitment boundaries.

\begin{figure}[t]
\centering
\includegraphics[width=\linewidth]{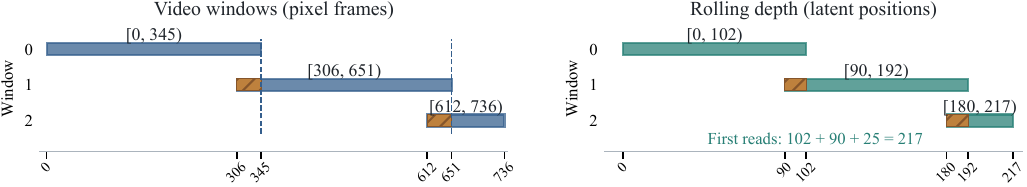}
\caption{\textbf{Aligned visual and geometric state.} Hatched intervals denote inherited history. Depth is read exactly once over \([0,102)\), \([102,192)\), and \([192,217)\). Dashed pixel lines mark the nominal latent-commitment boundaries, while stable incremental pixel output retains an additional five-frame blend tail.}
\label{fig:timeline}
\end{figure}

\subsection{Rolling Depth Latent State}
\label{sec:rolling}

For $i>0$, let $h_i=b_{i-1}-a_i$ denote the overlap with the preceding window. Rather than reprocessing overlapping geometry, we retain the corresponding canonical depth latents and append only newly requested positions:
\begin{equation}
 \boxed{
 C_i=
 \concat\left(
 C_{i-1}[a_i-a_{i-1}:L_{i-1}],
 \;\readdepth(b_{i-1},b_i)
 \right).
 }
 \label{eq:rolling}
\end{equation}
The first window reads $[a_0,b_0)$. A state update is committed only after interval, shape, dtype, and finite-value checks pass; retrying a window that has already been committed returns the cached condition. Once constructed, the depth state remains fixed throughout all $K$ model evaluations:
\begin{equation}
 C_i^{(0)}=C_i^{(1)}=\cdots=C_i^{(K-1)}=C_i.
 \label{eq:fixed_depth}
\end{equation}
Thus, unlike video history, depth does not follow the flow noise level; it remains a clean FP32 structural condition. Nevertheless, features that jointly depend on depth, the noisy video state, and the current flow time are recomputed at every model evaluation.

Under sequential, gap-free traversal and correct source reads, three properties follow (Appendix~\ref{app:proofs}). First,
\begin{equation}
 C_i=C[a_i:b_i],
\end{equation}
so overlapping depth positions retain exactly the same canonical latent values and dtype across windows. Second, strictly increasing right endpoints produce disjoint first-read intervals with a total temporal extent of $F$, independent of the number of flow evaluations $K$. Therefore, each global depth position is introduced into the rolling state exactly once. Third, the active depth rows satisfy
\begin{equation}
 \boxed{
 M_{\mathrm{depth}}
 \leq
 4L_{\max}S\times96
 \ \mathrm{bytes}.
 }
 \label{eq:memory}
\end{equation}
For the default $L_{\max}=102$, this yields an upper bound of $19.050\miB$ at $960\times544$ resolution and $9.413\miB$ at $672\times384$. The bound covers only the active canonical depth rows and excludes transient copies, fused conditioning features, model activations, and storage owned by the underlying geometry source. Relative to direct slicing of the same canonically encoded source, rolling depth preserves the conditioning tensor exactly. Its role is therefore to organize the active state and avoid redundant access rather than to alter the model output.

\subsection{Joint Generation Algorithm}
\label{sec:algorithm}

The logical state of window $i$ is
\begin{equation}
 \mathcal S_i=
 \left(
 \Omega_i,H_i,C_i,I_i,\epsilon_i,c_{\mathrm g}
 \right).
\end{equation}
A single global temporal plan determines the indices of video history, depth state, and noise, ensuring that all three refer to the same latent-time positions. Algorithm~\ref{alg:uniworld} combines rolling depth, noise-aligned history preservation, and boundary-image updates into a unified continuous generation procedure.

\begin{algorithm}[t]
\caption{UniWorld continuous generation}
\label{alg:uniworld}
\begin{tabbing}
\hspace{1em}\=\hspace{1em}\=\kill
\textbf{Input:} $I_{\mathrm{ref}},p$, depth source $\mathcal D$, windows $\{\Omega_i\}$, global noise $\epsilon$, grid $\{t_k\}_{k=0}^{K}$\\
$c_{\mathrm g}\leftarrow E_{\mathrm{Qwen}}(I_{\mathrm{ref}},p)$; $I_0\leftarrow I_{\mathrm{ref}}$; $Z\leftarrow[\,]$; depth state $\leftarrow[\,]$\\
\textbf{for} each $\Omega_i=[a_i,b_i)$ \textbf{do}\\
\> $C_i\leftarrow\operatorname{RollDepth}(\mathcal D,\Omega_i)$\\
\> $H_i\leftarrow Z[a_i:a_i+h_i]$; $\epsilon_i\leftarrow\epsilon[a_i:b_i]$\\
\> $\mathcal C_i\leftarrow\bigl(C_i,\operatorname{EncodeBoundary}(I_i),c_{\mathrm g}\bigr)$\\
\> $z_i\leftarrow\proj_{i,t_0}(\epsilon_i)$\\
\> \textbf{for} $k=0,\ldots,K-1$ \textbf{do}\\
\>\> $v_i\leftarrow v_\theta(z_i,t_k;\mathcal C_i)$\\
\>\> $z_i\leftarrow\proj_{i,t_{k+1}}\!\left(z_i+(t_{k+1}-t_k)v_i\right)$\\
\> \textbf{end for}\\
\> $Z\leftarrow\concat\bigl(Z,z_i[h_i:L_i]\bigr)$; depth state $\leftarrow C_i$\\
\> \textbf{if} another window remains: $I_{i+1}\leftarrow\operatorname{DecodeAtNextStart}(Z)$\\
\textbf{end for}\\
Decode $Z$ under the global VAE phase and trim alignment padding to the requested delivery length.
\end{tabbing}
\end{algorithm}

%
% ===== Begin inlined: experiments.tex =====
\section{Experiments}
\label{sec:experiments}

\subsection{Experimental Setup}
\label{sec:setup}
\paragraph{Data and implementation.}
The primary set contains 100 thirty-second clips, with 25 each from driving, human activity, simulation, and synthetic scenes. Video Depth Anything (VDA) ViT-L FP32 estimates~\citep{chen2025videodepth} are prepared at $832\times480$, 16 fps, normalized over each trajectory, and encoded for $960\times544$ generation. GT follows source timestamps at 24 fps and the same canvas transform for fidelity evaluation; real-time throughput is evaluated against a 14-FPS target. Additional studies cover short-video quality, 24 sources under nine paired conditions, 25 native-depth inputs, 30 long videos, and three PyBullet conditions. Appendix~\ref{app:evaluation} records source identities, splits, and configured-manifest sensitivity analyses.

The generator combines the released MiniMax-H3 ControlNet-Union for native depth conditioning with the VDN-H3 8-step model obtained by DMD distillation. The Union branch connects to five of 50 H3 blocks; inference uses eight evaluations per window, shifts 12/3, and guidance 1. A 30-second output spans three windows and 720 RGB frames. Each worker uses eight H100 GPUs with FSDP/Ulysses, with no additional joint training of the released components.

\paragraph{Metrics and statistics.}
We report video-level RGB PSNR, SSIM~\citep{wang2004ssim}, and LPIPS~\citep{zhang2018lpips}. Geometry uses VDA rank correlation, depth-edge F1, and NMAE with calibration fixed after five seconds; native comparisons share inverse-metre GT and original valid masks. Temporal error measures disagreement with GT frame changes at fixed seams. Intervals use 10,000 category-stratified source-group bootstrap replicates, one fixed seed per source, and marginal 95\% intervals. Appendix~\ref{app:evaluation} specifies sampling, calibration, and aggregation.

\subsection{Video Quality Comparison}
\label{sec:vbench}
We compare the H3/VDN DMD8 generator used by \method{} with H3 50-step, LongVie2, and LongVie AR2 on the same 100 sources. Table~\ref{tab:vbench} averages 5- and 14-second outputs equally (200 videos per configuration), using VBench's seven quality dimensions and officially weighted Quality Score~\citep{huang2023vbench}. DMD8 achieves the highest Quality Score, background consistency, motion smoothness, aesthetic quality, and imaging quality, exceeding H3 50-step by 0.67 Quality Score points. Appendix~\ref{app:vbench} gives native configurations, duration-level results, and paired intervals.

\begin{table}[t]
\centering
\caption{VBench quality ($\times100$, $\uparrow$); best in bold.}
\label{tab:vbench}
{\scriptsize\setlength{\tabcolsep}{3pt}%
% ===== Begin inlined: tables/vbench_quality.tex =====
% Generated by scripts/update_vbench_tables.py; do not edit numerical rows.
\begin{adjustbox}{max width=\linewidth}
\begin{tabular}{@{}lrrrrrrrr@{}}
\toprule
Configuration & \shortstack{Quality\\Score} & \shortstack{Subject\\cons.} & \shortstack{Background\\cons.} & \shortstack{Temporal\\flickering} & \shortstack{Motion\\smoothness} & \shortstack{Dynamic\\degree} & \shortstack{Aesthetic\\quality} & \shortstack{Imaging\\quality} \\
\midrule
H3 50-step & 79.87 & \textbf{87.46} & 92.94 & \textbf{96.72} & 98.47 & 83.50 & 52.60 & 62.28 \\
LongVie2 & 78.54 & 86.90 & 92.86 & 96.46 & 97.83 & 85.00 & 50.50 & 58.64 \\
LongVie AR2 & 78.62 & 86.28 & 92.19 & 96.63 & 97.84 & \textbf{86.00} & 50.06 & 60.24 \\
H3/VDN DMD8 & \textbf{80.54} & 87.20 & \textbf{93.10} & 96.54 & \textbf{98.68} & 85.50 & \textbf{54.57} & \textbf{63.51} \\
\bottomrule
\end{tabular}
\end{adjustbox}
% ===== End inlined: tables/vbench_quality.tex =====
}
\end{table}

\subsection{Multi-Scenario Continuous Generation}
\label{sec:multi}
One checkpoint completes all 100 clips and 300 windows. Table~\ref{tab:quality} reports source fidelity and depth consistency across the four categories; overall PSNR is 13.3973 dB (95\% CI $[12.9211,13.8794]$). Figure~\ref{fig:examples} shows outputs from the first frozen ablation identity in each category. Appendix~\ref{app:evaluation} provides full intervals and continuation-only measurements.

\begin{table}[t]
\centering
\caption{Thirty-second fidelity and depth consistency, $960\times544$, $W/O=345/39$, 8 NFE. Rank and Edge denote depth Spearman and edge F1; NMAE fixes calibration after five seconds. The final row excludes configured-manifest overlap.}
\label{tab:quality}
{\small\setlength{\tabcolsep}{3.2pt}%
% ===== Begin inlined: tables/quality.tex =====
% Generated by scripts/update_experiment_assets.py; do not edit numerical rows.
\begin{adjustbox}{max width=\linewidth}
\begin{tabular}{@{}lrrrrrrr@{}}
\toprule
Set & $n$ & PSNR$\uparrow$ & SSIM$\uparrow$ & LPIPS$\downarrow$ & Rank$\uparrow$ & Edge$\uparrow$ & NMAE$\downarrow$ \\
\midrule
Driving & 25 & 13.4429 & 0.5109 & 0.4727 & 0.9042 & 0.7027 & 0.0554 \\
Human activity & 25 & 14.9858 & 0.4799 & 0.3883 & 0.9299 & 0.7509 & 0.0523 \\
Simulation & 25 & 12.4034 & 0.3757 & 0.5914 & 0.7665 & 0.5223 & 0.2524 \\
Synthetic & 25 & 12.7571 & 0.3612 & 0.5475 & 0.9117 & 0.5655 & 0.0877 \\
\midrule
All & 100 & 13.3973 & 0.4319 & 0.5000 & 0.8781 & 0.6354 & 0.1120 \\
No manifest overlap & 90 & 13.1587 & 0.4357 & 0.5127 & 0.8702 & 0.6212 & 0.1181 \\
\bottomrule
\end{tabular}
\end{adjustbox}
% ===== End inlined: tables/quality.tex =====
}
\end{table}

\begin{figure}[t]
\centering
\includegraphics[width=\linewidth]{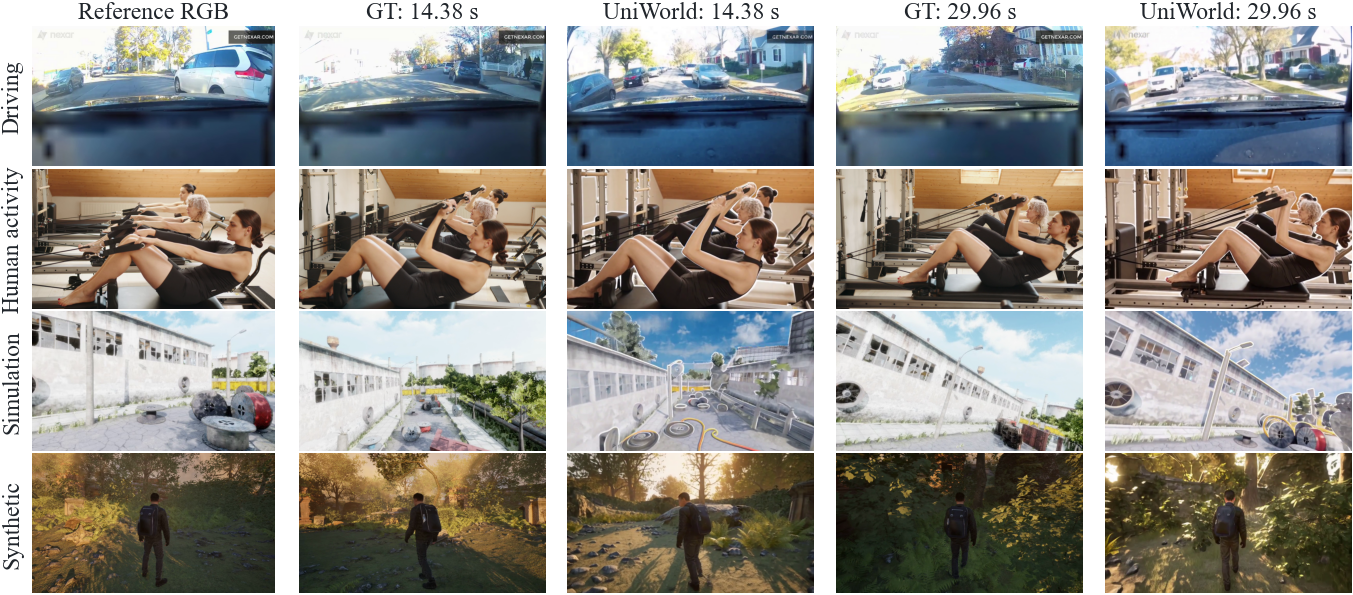}
\caption{\textbf{One model across four scene categories.} Actual reference, GT, and generated frames at the first continuation boundary and final frame. Rows use the first frozen ablation identity per category; selection is independent of quality scores.}
\label{fig:examples}
\end{figure}

\paragraph{Native depth.}
The same model also completes all 25 PointOdyssey/TartanAir inputs with native metric depth by changing only the source adapter. Appendix~\ref{app:depth_protocol} and Table~\ref{tab:native} give the matched comparison against common native GT, paired intervals, and valid-support rules.

\subsection{Mechanism Analysis and Ablations}
\label{sec:ablations}
\paragraph{State consistency in real model runs.}
Full-method replay, rolling/direct reads, and aligned depth re-encoding match exactly in all 24 cases; all three online/cached pairs also match. Appendix~\ref{app:checks} and Table~\ref{tab:checks} retain the complete state checks, including committed history under a control branch.

\paragraph{Fixed-window comparisons.}
We select six distinct source groups per category by a fixed deterministic ordering. All nine conditions share the exact first-window output, suffix noise, reference/text, window geometry, commitment boundaries, schedule, and VAE. Partial-history variants constrain the final 6 or 2 of the 12 overlap latents, leaving the discarded overlap unchanged. Table~\ref{tab:paired} reports each alternative minus full per-step projection.
Full projection reduces seam LPIPS by 0.00195 relative to boundary RGB only in this paired sample. Continuation PSNR and seam temporal-error intervals include zero, and fresh history noise gives slightly lower seam LPIPS. The endpoint-only control matches boundary-only output exactly: both have unconstrained history during model evaluations and discard the predicted overlap. This isolates per-step history participation from final output preservation. Appendix~\ref{app:paired} provides continuation, geometry, and manifest-exclusion results under the same protocol.

\begin{table}[t]
\centering
\caption{Paired effects, 24 sources per arm; brackets are marginal 95\% bootstrap intervals. PSNR covers $[345,720)$; LPIPS and temporal error cover fixed seams. Positive LPIPS/error differences favor the full method.}
\label{tab:paired}
{\scriptsize\setlength{\tabcolsep}{2.5pt}%
% ===== Begin inlined: tables/paired.tex =====
% Generated by scripts/update_experiment_assets.py; do not edit numerical rows.
\begin{adjustbox}{max width=\linewidth}
\begin{tabular}{@{}lccc@{}}
\toprule
Arm minus full & $\Delta$PSNR (dB)$\uparrow$ & $\Delta$LPIPS$\downarrow$ & $\Delta E_{\rm temp}\downarrow$ \\
\midrule
Boundary RGB only & $0.00637\;[-0.01467,0.02591]$ & $0.00195\;[0.00033,0.00362]$ & $-0.00047\;[-0.00140,0.00039]$ \\
Endpoint only & $0.00637\;[-0.01467,0.02591]$ & $0.00195\;[0.00033,0.00362]$ & $-0.00047\;[-0.00140,0.00039]$ \\
Fresh history noise & $-0.01337\;[-0.08561,0.05934]$ & $-0.00198\;[-0.00344,-0.00045]$ & $0.00021\;[-0.00054,0.00103]$ \\
Last 6 history latents & $0.00623\;[-0.00484,0.01751]$ & $0.00030\;[-0.00010,0.00072]$ & $0.00002\;[-0.00022,0.00029]$ \\
Last 2 history latents & $0.00917\;[-0.00888,0.02600]$ & $0.00113\;[-0.00016,0.00261]$ & $-0.00023\;[-0.00056,0.00008]$ \\
Direct global slice & $0.00000\;[0.00000,0.00000]$ & $0.00000\;[0.00000,0.00000]$ & $0.00000\;[0.00000,0.00000]$ \\
Window normalization & $-0.00897\;[-0.04902,0.03094]$ & $0.00063\;[-0.00076,0.00200]$ & $0.00085\;[-0.00014,0.00210]$ \\
Aligned re-encoding & $0.00000\;[0.00000,0.00000]$ & $0.00000\;[0.00000,0.00000]$ & $0.00000\;[0.00000,0.00000]$ \\
\bottomrule
\end{tabular}
\end{adjustbox}
% ===== End inlined: tables/paired.tex =====
}
\end{table}

\subsection{Runtime and Resource Usage}
\label{sec:runtime}

\paragraph{DiT forward throughput.}
Four fixed sources at each resolution receive one full-window warmup and three hot repeats: 32 requests including warmups, 12 hot requests per canvas. We time actual H3/VDN and Union forwards, including FSDP/Ulysses communication, with CUDA synchronization and per-step maxima over eight ranks. The devices are isolated from other experiment jobs; final latents match the corresponding complete-generation runs.

We evaluate real-time throughput against a 14-FPS target. Full-continuation windows produce 14.472 and 21.361 new FPS at $960\times544$ and $672\times384$, respectively, meeting the target at both resolutions. These timings cover synchronized H3/VDN and Union forwards with FSDP/Ulysses communication on eight H100 GPUs; Appendix~\ref{app:runtime} reports the complete window profiles and whole-trajectory throughput.

\paragraph{Depth-state resources.}
Active depth rows occupy 19.050 MiB at $960\times544$ and 9.413 MiB at $672\times384$, independently of trajectory length. Rolling reads 217/427/852 positions for 30/60/120-second outputs, versus 241/475/960 for direct slicing, while preserving identical output. This provides a bounded geometric working state and avoids repeated source requests. Consumer copies, carrier, hints, model storage, and growing global caches are accounted for separately, together with measured reader-switch costs, in Appendix~\ref{app:resources}.

\subsection{Long-Horizon Behaviour and Applications}
\label{sec:long}
The same window interface advances through 17 sixty-second videos from existing validation/test splits and thirteen 120-second training-split stress tests. Their full-video PSNR/LPIPS are 14.4542/0.49032 and 14.4277/0.54645. Figure~\ref{fig:long_horizon} characterizes fidelity and fixed-calibration geometric agreement over successive ten-second blocks; each duration uses its own source set and source-group confidence bands.

\begin{figure}[t]
\centering
\includegraphics[width=\linewidth]{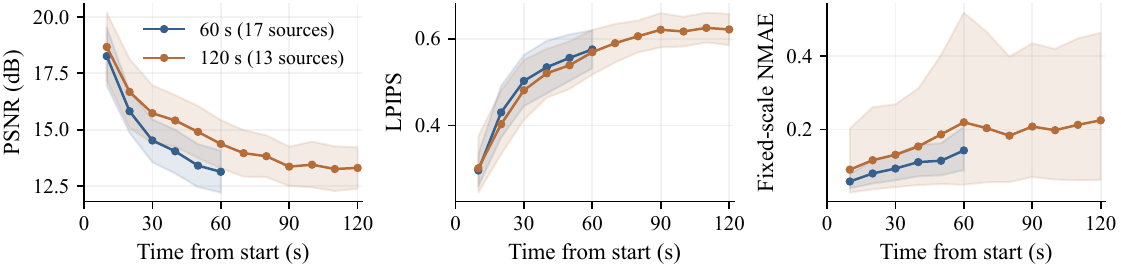}
\caption{\textbf{Continuous generation over 60--120 seconds.} Ten-second aggregates with category-stratified source-group 95\% bootstrap bands. The duration groups use different sources; the 120-second group is a training-split stress test.}
\label{fig:long_horizon}
\end{figure}

Three procedural PyBullet conditions~\citep{coumans2021pybullet} exercise warm/cool reference appearance, a camera revisit at frame 480, and a control branch activated at the second-window request, affecting frame 357 onward. On-demand and cached depth give bitwise-identical final latents in all three pairs. The control branch preserves the first window exactly and changes the new suffix (latent mean absolute difference 0.039076). Warm/cool references share identical geometry while producing different appearance. The first request reads 357 geometry frames, and RGB is delivered incrementally with the canonical decoder phase. Appendix~\ref{app:online} reports the appearance, revisit, and response measurements for this window-level engine interface.

% ===== End inlined: experiments.tex =====

\section{Conclusion}
\label{sec:conclusion}
We presented \method{}, a unified framework for continuous visual world simulation from depth trajectories and reference appearance. Its joint state contract preserves committed video latents, retains matching geometric controls, and reads new depth positions incrementally with bounded active state. One shared generator connects estimated depth, native simulation data, and an on-demand engine; exact replay and cached/online equivalence verify this connection in real model runs. Multi-scenario generation, fixed-window interventions, and long-horizon trajectories characterize the resulting visual simulation, while stage-specific measurements make its computation and data-access costs explicit. The measured steady-state DiT path achieves real-time throughput while maintaining bounded active state for continuous generation.
\label{main:last}

\FloatBarrier
\bibliographystyle{plainnat}
\bibliography{references}

\clearpage
\appendix
\setcounter{figure}{0}
\setcounter{table}{0}
\setcounter{equation}{0}
\setcounter{proposition}{0}
\setcounter{algorithm}{0}
\renewcommand{\thefigure}{S\arabic{figure}}
\renewcommand{\thetable}{S\arabic{table}}
\renewcommand{\theequation}{S\arabic{equation}}
\renewcommand{\theproposition}{S\arabic{proposition}}
\renewcommand{\thealgorithm}{S\arabic{algorithm}}
\renewcommand{\theHfigure}{supp.\arabic{figure}}
\renewcommand{\theHtable}{supp.\arabic{table}}
\renewcommand{\theHequation}{supp.\arabic{equation}}
\providecommand{\theHproposition}{}
\renewcommand{\theHproposition}{supp.\arabic{proposition}}
\providecommand{\theHalgorithm}{}
\renewcommand{\theHalgorithm}{supp.\arabic{algorithm}}
\input{appendix}

\end{document}

%% file: appendix.tex
% Supplementary material included by paper.tex.
\FloatBarrier
\section{Proofs and State Invariants}
\label{app:proofs}

\subsection{History Projection}
We regard the binary mask $M_i$ as broadcast over channel and spatial dimensions. The saved history $H_i$ and global noise slice $\epsilon_i$ are fixed within a window. All values are finite; statements below are exact in real arithmetic. Define $Q_{i,t}=tH_i+(1-t)\epsilon_i$.

\begin{proposition}[Idempotence and per-step history constraint]
For fixed $i,t$, $\proj_{i,t}$ is idempotent. The projected sampler satisfies $M_i\odot z_i^{(k)}=M_i\odot Q_{i,t_k}$ for every $k=0,\ldots,K$, and its terminal history is $H_i$.
\end{proposition}
\begin{proof}
Since $M_i^2=M_i$ and $M_i\odot(1-M_i)=0$,
\begin{align}
 \proj_{i,t}(\proj_{i,t}(z))
 &=M_i\odot Q_{i,t}+(1-M_i)\odot\left[M_i\odot Q_{i,t}+(1-M_i)\odot z\right]\\
 &=M_i\odot Q_{i,t}+(1-M_i)\odot z=\proj_{i,t}(z).
\end{align}
Initialization projects at $t_0$, establishing the constraint for $k=0$. Each subsequent update ends by applying $\proj_{i,t_{k+1}}$, which replaces every masked entry by the corresponding entry of $Q_{i,t_{k+1}}$, regardless of the predicted velocity. This proves the induction step. Finally $t_K=1$, so $Q_{i,t_K}=H_i$ on the prefix.
\end{proof}

Equation~\ref{eq:commit} separately preserves the stored global prefix by concatenating only new positions. Thus global committed history does not depend on recovering the sampled prefix at the end. Per-step projection serves the distinct purpose of placing that history in every model evaluation. At $h_i=0$, $M_i=0$ and the projection is the identity. We require $h_i<L_i$, so each continuation adds content.

\paragraph{Finite precision.}
The implementation owns FP32 copies of history and noise and writes $tH_i+(1-t)\epsilon_i$ directly into prefix rows. Repeated projections at the same numerical $t$ use the same operands. At $t=1$, the finite prefix values are recovered without a dtype conversion; the stored committed prefix is never rewritten. Claims of bitwise equality refer to the actual tensor checks under this implementation. They do not imply exact real-valued flow integration, equality under arbitrary precision changes, or equality of RGB decoded with different temporal contexts. The operation changes a sampling boundary condition, not the learned velocity field, and does not guarantee a globally coherent scene representation.

\subsection{Depth Overlap and Incremental Reads}
Let a canonical depth sequence $C$ have length $F$, and let the windows satisfy $a_0=0$, $b_{J-1}=F$, $a_i<a_{i+1}\leq b_i<b_{i+1}$, and $0\leq h_{i+1}<L_{i+1}$. Each read returns the requested canonical values in the same dtype and order.

\begin{proposition}[Canonical-condition equivalence]
If $C_0=C[a_0:b_0]$, equation~\ref{eq:rolling} gives $C_i=C[a_i:b_i]$ for every window. In particular, $C_i[0:h_i]=C_{i-1}[a_i-a_{i-1}:L_{i-1}]$.
\end{proposition}
\begin{proof}
The base case is assumed. If the statement holds for $i-1$, its retained suffix is exactly $C[a_i:b_{i-1}]$. The new read is $C[b_{i-1}:b_i]$. These intervals meet at $b_{i-1}$ without a gap or duplicated element. Their concatenation is $C[a_i:b_i]$, proving the induction step. Identical-value copying preserves the overlapping FP32 bit patterns.
\end{proof}

\begin{proposition}[Read count and active storage]
During one successful sequential traversal, first reads cover $F$ temporal positions, and one active depth state uses at most $4L_{\max}S\times96$ bytes.
\end{proposition}
\begin{proof}
The intervals are $[0,b_0),[b_0,b_1),\ldots,[b_{J-2},F)$. They are disjoint and partition $[0,F)$, giving $b_0+\sum_{i=1}^{J-1}(b_i-b_{i-1})=F$ positions. No model evaluation reads an additional interval because $C_i$ is fixed within the window. The active tensor contains at most $L_{\max}S$ rows, each with 96 FP32 elements of four bytes.
\end{proof}

First-read counts concern logical source positions, not operating-system disk traffic or GPU copies. A direct source slice can be a view; it need not perform physical I/O. Concatenation, defensive return copies, and carrier packing still cost time and memory. A failed read may be retried, and restarting a request may reread its source. The stated count applies to successful traversal and same-window cache reuse, without such restarts. Likewise, transient old-plus-new buffers can exceed the size of a single active state.

\FloatBarrier
\section{Conditioning and Implementation}
\label{app:conditioning}

\subsection{Model Components}
The backbone is MiniMax-H3 with VDN hybrid video attention~\citep{minimax2026h3,xi2026vdn}. The 50-block backbone retains bidirectional within-window processing. The five-block dense Union branch comes from MiniMax-H3-Fun-Controlnet-Union in VideoX-Fun~\citep{aigcapps2026videox}; its parameters, Q/K/V projections, and activations remain separate. The original Qwen3-VL-32B encoder~\citep{bai2025qwen3vl} provides layer-50 hidden states of width 5,120 for image and text conditioning. The existing video VAE supplies both the RGB and depth encodings.

The shared generator combines published VDN weights with the existing H3 and Union components, without additional local joint training or a new distillation objective. Appendix~\ref{app:evaluation} describes evaluation-source selection and the scope of the local training-overlap analysis.

Experiments use PyTorch 2.11 with CUDA 12.8 on workers equipped with eight 80-GB NVIDIA H100 GPUs. Metric and simulator versions are specified with their evaluation protocols.

\subsection{Independent First-Frame Conditions}
The original Qwen representation stays fixed across windows. The single-frame VAE path uses posterior sampling with encoding seed 42, FP16 rounding, then FP32 latent normalization. Denote the resulting row tensor by $R_i$. The fixed first-frame input is
\begin{equation}
 c_i^{\mathrm{ref}}=0.999R_i+0.001\epsilon^{\mathrm{ref}},
 \qquad t_i^{\mathrm{ref}}=\max(t,0.999).
 \label{eq:keyframe}
\end{equation}
The noise realization is fixed for the request. The input rows remain fixed within a window even though their model activations change. This differs from target-prefix projection, whose values explicitly follow $t$ from 0 to 1. Casting 0.999 to BF16 first would round it to 1.0; the prescribed condition mixture and history projection therefore use FP32 arithmetic. The model retains its native modulation conventions. We do not refresh the Qwen embedding from each generated boundary or use GT boundaries during inference.

\subsection{Depth Rows and Union Packing}
Depth uses posterior \emph{mode}, followed by FP32 normalization and $(1,2,2)$ patchification. Unlike equation~\ref{eq:keyframe}, depth is clean conditioning without the first-frame mixture. The Union input has 49 pre-patch channels: 24 for depth, one for a visibility/mask channel, and 24 for inpainting. Spatial patchification gives $49\times4=196$ carrier columns. The 96 depth columns occupy the target-video rows; the remaining 100 columns are zero for this path. A first-frame placeholder is included in the carrier layout.

Target rows are selected by explicit indices. Qwen vision tokens, independent VAE condition rows, and target video rows can share modality tags, so a tag test alone cannot identify the target. Zero padding must also be combined with explicit row selection to limit the effect of a learned projection bias. The depth tensor and control carrier are reused within a window, but Union hints depend on the current noisy target and time and are recomputed at every step. There is no cross-window backbone K/V cache introduced by the two state rules.

\begin{table}[htbp]
\centering
\caption{Analytical storage for FP32 depth rows. Carrier storage, when listed, includes one first-frame placeholder; transient copies and model state are excluded.}
\label{tab:storage}
\begin{adjustbox}{max width=\linewidth}
\begin{tabular}{lrrrr}
\toprule
Canvas & $S$ & 102-latent depth & 12-latent overlap & 196-column carrier \\
\midrule
$672\times384$ & 252 & $9.413\miB$ & $1.107\miB$ & $19.407\miB$ \\
$960\times544$ & 510 & $19.050\miB$ & $2.241\miB$ & $39.276\miB$ \\
\bottomrule
\end{tabular}
\end{adjustbox}
\end{table}

\subsection{Sampling and State Lifecycle}
Each video window uses eight actual velocity evaluations with video shift 12; the audio shift is 3 and guidance scale is 1. A scheduler interface that counts grid points may expose nine points for eight updates. Global video and audio noise and first-frame noise are prepared deterministically and sliced consistently. Projection runs once before the first evaluation and once after each update, so the next model input already satisfies the constraint. The first-window mask is empty.

Rolling depth validates ordered window indices, global starts, VAE/audio phase, overlap, shape, FP32 dtype, and finite newly read values. Valid data are committed atomically to the logical cache; returning a copy prevents callers from modifying retained controls through shared storage. Retrying the current window reuses that state.

%
% ===== Begin inlined: appendix_experiments.tex =====
\FloatBarrier
\section{Evaluation Protocols and Additional Comparisons}
\label{app:evaluation}

\subsection{Datasets and Source Selection}
The primary set comprises the sources in Table~\ref{tab:data}. Named collections include DanceTrack~\citep{sun2022dancetrack}, BEHAVE~\citep{bhatnagar2022behave}, PointOdyssey~\citep{zheng2023pointodyssey}, TartanAir~\citep{wang2020tartanair}, and MIND~\citep{ye2026mind}. We use PointOdyssey v1.2 and TartanAir V2 as part of this common evaluation set rather than their official benchmark protocols. The human-activity category also includes eleven local original clips.

\begin{table}[htbp]
\centering
\caption{Primary 30-second sources. Every category initially uses estimated VDA depth. Native-depth inputs are a paired intervention on the same 25 simulation sources.}
\label{tab:data}
\begin{tabularx}{\linewidth}{@{}lrX@{}}
\toprule
Category & Clips & Collection \\
\midrule
Driving & 25 & Nexar \\
Human activity & 25 & DanceTrack (10), BEHAVE (4), local originals (11) \\
Simulation & 25 & PointOdyssey v1.2 (15), TartanAir V2 (10) \\
Synthetic & 25 & MIND (25) \\
\bottomrule
\end{tabularx}
\end{table}

The 100 primary clips contain 96 distinct source groups. Ten human-activity clips also appear in the configured 10,000-sample local training set. The evaluated generator reuses pretrained components without local joint training; exclusion of these clips is therefore reported as a sensitivity analysis. This overlap check covers the local data collection and does not characterize the original pretraining data of H3, VDN, or Union.

The 24 ablation sources comprise six distinct source groups per category, selected by a fixed deterministic ordering independent of quality. Every source uses one deterministic seed. The 17 sixty-second sources use existing validation/test splits; all 13 120-second sources belong to an existing training split and are reported as stress tests.

\subsection{Depth Representation and Ground-Truth Alignment}
\label{app:depth_protocol}
Estimated depth uses VDA ViT-L FP32 over the full trajectory at $832\times480$, 16 fps. Normalized depth is temporally resampled and encoded on the generation canvas. RGB GT is resampled at 24 fps using source timestamps, with an $832\times480$ cover resize/crop followed by $960\times544$ Lanczos resizing for primary comparisons. This 24-FPS stream is retained for fidelity evaluation; real-time throughput is evaluated against a 14-FPS target. Generated videos are evaluated at their delivered resolution and duration, including the generated first frame. Padding needed for $17n+5$ geometry is never included in delivered-frame counts.

\paragraph{Verified numerical direction.}
VDA raw output is inverse-depth oriented, with larger values near the camera. The normalization helper reverses that ordering, making the actual condition \emph{far bright}. For metric depth $z>0$, native preparation uses $u=1/z$, obtains trajectory-wide $q_5,q_{95}$ from originally valid inverse-depth pixels, and applies
\begin{equation}
 \bar D=\frac{q_{95}-\operatorname{clip}(u,q_5,q_{95})}{q_{95}-q_5}.
 \label{eq:native_mapping}
\end{equation}
This matches the checkpoint's existing numeric convention; it does not change the VDA baseline. Native TartanAir uses packed float32 metres; PointOdyssey depth uses $z=1000\,d_{16}/65535$ metres. Invalid pixels are filled temporally for generation only, and original masks accompany the metrics. All 25 native comparisons retain identical prompt embeddings, text tags, first-frame rows, reference, noise identity, and model; only the depth condition changes.

\begin{table}[htbp]
\centering
\caption{Estimated versus native depth on 25 matched sources; geometry uses common native GT.}
\label{tab:native}
{\small\setlength{\tabcolsep}{3.8pt}%
% ===== Begin inlined: tables/native.tex =====
% Generated by scripts/update_experiment_assets.py; do not edit numerical rows.
\begin{adjustbox}{max width=\linewidth}
\begin{tabular}{@{}lrrrrrr@{}}
\toprule
Input & PSNR$\uparrow$ & SSIM$\uparrow$ & LPIPS$\downarrow$ & Rank$\uparrow$ & Edge$\uparrow$ & NMAE$\downarrow$ \\
\midrule
Estimated VDA & 12.4034 & 0.3757 & 0.5914 & 0.7818 & 0.4951 & 0.2783 \\
Native metric depth & 11.6948 & 0.3527 & 0.6208 & 0.7778 & 0.4798 & 0.2722 \\
\bottomrule
\end{tabular}
\end{adjustbox}
% ===== End inlined: tables/native.tex =====
}
\end{table}

\begingroup
\setlength{\emergencystretch}{3em}
For completeness, the native-minus-estimated paired differences are $-0.70859$ dB PSNR (95\% CI $[-1.31984,-0.04591]$), $+0.02943$ LPIPS ($[0.00569,0.05236]$), $-0.00408$ native-GT rank ($[-0.03886,0.04161]$), and $-0.00612$ native-GT NMAE ($[-0.05175,0.03045]$). These use the same 25 sources and 21 source groups, with common native GT for both geometry evaluations.
\par
\endgroup

\paragraph{On-demand source.}
PyBullet 3.2.7 TinyRenderer converts its z-buffer to axial metre depth using near/far planes 0.1/25 m. A fixed, predeclared 1.5/12 m mapping supplies far-bright conditions:
\begin{equation}
 \bar D_{\mathrm{engine}}(z)=\operatorname{clip}\!\left(
 \frac{1/1.5-1/z}{1/1.5-1/12},0,1\right).
\end{equation}
It is quantized to uint8 and encoded in canonical blocks. This adapter uses fixed metric bounds and does not inspect future trajectory percentiles. Warm and cool appearances have bitwise-identical depth, and the control branch shares the initial source interval exactly.

\subsection{Fidelity, Perceptual, and Temporal Metrics}
For video $s$, RGB values in $[0,1]$, and frame set $\mathcal T$, the primary PSNR is
\begin{align}
 \operatorname{MSE}_s(\mathcal T)
 &=\frac{1}{3H_{\mathrm{px}}W_{\mathrm{px}}|\mathcal T|}
   \sum_{j\in\mathcal T}\sum_{x,y,c}(\hat V_{s,j,x,y,c}-V_{s,j,x,y,c})^2,\\
 \operatorname{PSNR}_s(\mathcal T)&=-10\log_{10}\operatorname{MSE}_s(\mathcal T).
\end{align}
Decoded eight-bit RGB is converted to FP64 for this calculation. We average video-level dB, not framewise dB or the dB of pooled cross-video MSE. SSIM uses FFmpeg's default $8\times8$ blocks over all RGB channels, averaged across frames then videos. Full, first-frame-excluded, and common continuation sets are $[0,720)$, $[1,720)$, and $[345,720)$.

LPIPS uses the official v0.1 AlexNet model (package 0.1.4), RGB scaled to $[-1,1]$, at the original inference canvas. Its deterministic sample contains every twelfth frame at 24 fps plus all frames in the fixed seam bands. The full LPIPS mean therefore includes this seam-augmented sample; seam LPIPS is reported separately. Temporal error uses every 24-fps consecutive-frame transition:
\begin{equation}
 E_{\mathrm{temp}}(j)=\operatorname{mean}_{x,y,c}
 \left| (\hat V_j-\hat V_{j-1})-(V_j-V_{j-1})\right|.
 \label{eq:temporal}
\end{equation}
Prediction and GT motion magnitudes are recorded separately. Small frame-to-frame change alone does not establish correct motion. Fixed 12-frame neighborhoods are $[339,351)$ and $[645,657)$; all ablations use identical absolute indices.

Across the primary set, SSIM is 0.431918, LPIPS is 0.49997, and depth rank correlation is 0.87809. Excluding frame zero gives 13.3915 dB PSNR; the common continuation $[345,720)$ gives 12.6700 dB PSNR and 0.54241 LPIPS.

\subsection{Geometric Metrics and Valid Support}
Generated RGB is sampled uniformly at 2 fps, resized to $832\times480$, and passed to VDA ViT-L FP32 with internal input size 518. Prediction and target depth are area-resized to $208\times120$ for metrics. Rank is Spearman correlation. Edge F1 thresholds each depth-gradient magnitude at its valid-region 90th percentile and permits a one-pixel match radius. A nonnegative affine scale and offset are fitted once over valid pixels in the first five seconds; later frames never refit. NMAE divides fixed-calibration absolute error by the calibration target's $p_{95}-p_5$.

Primary geometry compares generated-video VDA output against the VDA conditioning trajectory and measures estimator self-consistency. In native comparisons, both generated outputs use the \emph{same native inverse-metre target}; scores against different conditioning sources cannot establish source-to-source geometric gains. This native target removes dependence on VDA for target geometry, while the generated-side estimate still has estimator bias.

A downsampled native-depth pixel is valid only if every original pixel contributing to its area average is valid. Edge evaluation additionally requires a fully valid $3\times3$ Sobel neighborhood. Filled values never participate in calibration or evaluation. Frames without valid support or calibration contribute no value to the corresponding metric; aggregate means use the remaining valid observations and retain their sample counts.

\subsection{Cluster Bootstrap and Aggregate Intervals}
Bootstrap uses seed 20260924 and 10,000 replicates. We sample source groups with replacement within each category, preserving all videos and their weights in each selected group. For ablations, resampling uses paired same-source differences. This estimates variation across sources under a fixed seed per source, not seed-to-seed generation variability. Intervals are marginal 95\% percentiles, with no family-wise adjustment across the exploratory arms and metrics.

\begin{table}[htbp]
\centering
\caption{Primary source-fidelity means with category-stratified source-group 95\% bootstrap intervals. The 90-clip subset excludes overlap with the configured local training set.}
\label{tab:quality_ci}
{\small\setlength{\tabcolsep}{3pt}%
% ===== Begin inlined: tables/quality_intervals.tex =====
% Generated by scripts/update_experiment_assets.py; do not edit numerical rows.
\begin{adjustbox}{max width=\linewidth}
\begin{tabular}{@{}lccc@{}}
\toprule
Set & PSNR (dB)$\uparrow$ & SSIM$\uparrow$ & LPIPS$\downarrow$ \\
\midrule
Driving & $13.4429\;[12.4759,14.4184]$ & $0.5109\;[0.4682,0.5516]$ & $0.4727\;[0.4283,0.5158]$ \\
Human activity & $14.9858\;[13.9759,15.9918]$ & $0.4799\;[0.4163,0.5395]$ & $0.3883\;[0.3499,0.4316]$ \\
Simulation & $12.4034\;[11.2682,13.5804]$ & $0.3757\;[0.3093,0.4445]$ & $0.5914\;[0.5529,0.6262]$ \\
Synthetic & $12.7571\;[12.0658,13.4274]$ & $0.3612\;[0.3073,0.4178]$ & $0.5475\;[0.5238,0.5701]$ \\
All & $13.3973\;[12.9211,13.8794]$ & $0.4319\;[0.4037,0.4601]$ & $0.5000\;[0.4813,0.5185]$ \\
No manifest overlap & $13.1587\;[12.6365,13.6800]$ & $0.4357\;[0.4063,0.4653]$ & $0.5127\;[0.4926,0.5329]$ \\
\bottomrule
\end{tabular}
\end{adjustbox}
% ===== End inlined: tables/quality_intervals.tex =====
}
\end{table}

\subsection{Paired Mechanism Protocol and Additional Results}
\label{app:paired}
All 216 paired outputs use the same first-window video/audio latents for their source. Non-depth inputs, suffix noise, window size, discarded overlap, commitment points, 8 NFE, video/audio schedule 12/3, and VAE precision remain fixed. The depth interventions alter only the intended continuation conditions. Table~\ref{tab:arms} specifies the interventions.

\begin{table}[htbp]
\centering
\caption{Nine completed conditions, each applied to the same 24 sources. Geometry and commitment boundaries remain fixed.}
\label{tab:arms}
\begin{tabularx}{\linewidth}{@{}lX@{}}
\toprule
Condition & Intervention \\
\midrule
Full & Project all 12 history latents at every sampling step; canonical rolling depth. \\
Boundary RGB only & Retain independent boundary-image conditioning; no target-history projection. \\
Endpoint only & Project the full overlap only at $t=1$, after all model evaluations. \\
Fresh history noise & Redraw overlap noise once, then keep it fixed across steps; suffix noise is unchanged. \\
Last 6 / last 2 latents & Project only $[h_i-q,h_i)$ with $q=6$ or 2, still discarding all $h_i$ overlap latents. \\
Direct global slice & Read canonical $C[a_i:b_i]$ independently at each advance. \\
Window normalization & Recompute percentiles for continuation windows with canonical global encoding context. \\
Aligned re-encoding & Keep globally normalized pixels; encode each phase-aligned window independently with its padded tail. \\
\bottomrule
\end{tabularx}
\end{table}

At the endpoint, restoring history cannot alter the suffix already generated. Boundary-only and endpoint-only variants therefore produce identical committed outputs, as confirmed in 24/24 cases. Direct slicing and rolling retain identical conditions and outputs, also in 24/24 cases. Independent phase-aligned re-encoding is exact because the H3 encoder uses independent 17-frame chunks; these cases do not demonstrate encoding-context drift. Comparisons that vary both overlap and window size are reported separately in Appendix~\ref{app:overlap_results}.

\begin{figure}[t]
\centering
\includegraphics[width=\linewidth]{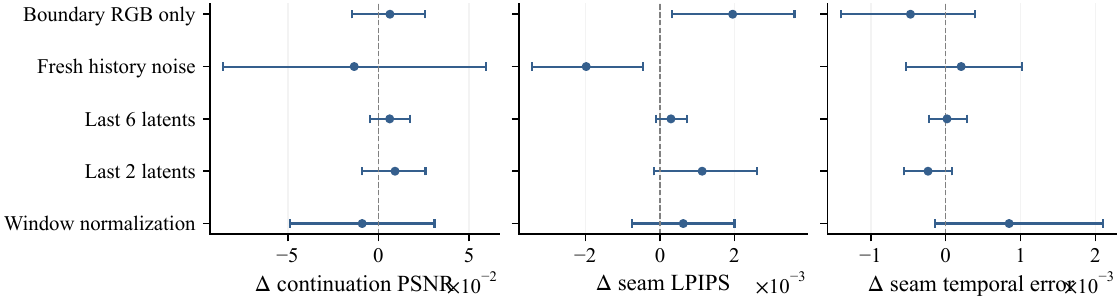}
\caption{Paired effect sizes and marginal 95\% intervals for non-equivalent interventions, each minus full history projection. Endpoint-only is identical to boundary-only; direct slicing and aligned re-encoding are exact zero-effect controls.}
\label{fig:paired_effects}
\end{figure}

\begin{table}[htbp]
\centering
\caption{Additional continuation metrics on $[345,720)$, alternative minus full, $n=24$ per row. Brackets give marginal 95\% source-group bootstrap intervals.}
\label{tab:paired_continuation}
{\scriptsize\setlength{\tabcolsep}{2.5pt}%
% ===== Begin inlined: tables/paired_continuation.tex =====
% Generated by scripts/update_experiment_assets.py; do not edit numerical rows.
\begin{adjustbox}{max width=\linewidth}
\begin{tabular}{@{}lccc@{}}
\toprule
Arm minus full & $\Delta$SSIM$\uparrow$ & $\Delta$LPIPS$\downarrow$ & $\Delta E_{\rm temp}\downarrow$ \\
\midrule
Boundary RGB only & $0.00018\;[-0.00052,0.00083]$ & $0.00157\;[0.00036,0.00288]$ & $-0.00032\;[-0.00045,-0.00020]$ \\
Endpoint only & $0.00018\;[-0.00052,0.00083]$ & $0.00157\;[0.00036,0.00288]$ & $-0.00032\;[-0.00045,-0.00020]$ \\
Fresh history noise & $-0.00357\;[-0.00861,0.00032]$ & $-0.00118\;[-0.00266,0.00045]$ & $0.00000\;[-0.00042,0.00039]$ \\
Last 6 history latents & $0.00045\;[0.00001,0.00083]$ & $0.00026\;[-0.00009,0.00064]$ & $-0.00001\;[-0.00009,0.00008]$ \\
Last 2 history latents & $0.00049\;[-0.00011,0.00103]$ & $0.00099\;[0.00001,0.00208]$ & $-0.00030\;[-0.00044,-0.00017]$ \\
Direct global slice & $0.00000\;[0.00000,0.00000]$ & $0.00000\;[0.00000,0.00000]$ & $0.00000\;[0.00000,0.00000]$ \\
Window normalization & $-0.00050\;[-0.00278,0.00222]$ & $0.00052\;[-0.00073,0.00182]$ & $0.00048\;[-0.00002,0.00110]$ \\
Aligned re-encoding & $0.00000\;[0.00000,0.00000]$ & $0.00000\;[0.00000,0.00000]$ & $0.00000\;[0.00000,0.00000]$ \\
\bottomrule
\end{tabular}
\end{adjustbox}
% ===== End inlined: tables/paired_continuation.tex =====
}
\end{table}

\begin{table}[htbp]
\centering
\caption{Full-video VDA geometry effects, alternative minus full, $n=24$ per row. These measure consistency with VDA input, not independent metric geometry. Intervals are unadjusted across metrics and arms.}
\label{tab:paired_geometry}
{\scriptsize\setlength{\tabcolsep}{2.5pt}%
% ===== Begin inlined: tables/paired_geometry.tex =====
% Generated by scripts/update_experiment_assets.py; do not edit numerical rows.
\begin{adjustbox}{max width=\linewidth}
\begin{tabular}{@{}lccc@{}}
\toprule
Arm minus full & $\Delta$Rank$\uparrow$ & $\Delta$Edge$\uparrow$ & $\Delta$NMAE$\downarrow$ \\
\midrule
Boundary RGB only & $-0.00091\;[-0.00158,-0.00031]$ & $-0.00078\;[-0.00181,0.00028]$ & $0.00032\;[-0.00023,0.00084]$ \\
Endpoint only & $-0.00091\;[-0.00158,-0.00031]$ & $-0.00078\;[-0.00181,0.00028]$ & $0.00032\;[-0.00023,0.00084]$ \\
Fresh history noise & $-0.00103\;[-0.00361,0.00093]$ & $0.00011\;[-0.00281,0.00295]$ & $-0.00042\;[-0.00164,0.00086]$ \\
Last 6 history latents & $-0.00021\;[-0.00138,0.00095]$ & $0.00024\;[-0.00060,0.00118]$ & $-0.00018\;[-0.00081,0.00033]$ \\
Last 2 history latents & $-0.00084\;[-0.00232,0.00016]$ & $-0.00083\;[-0.00153,-0.00014]$ & $-0.00011\;[-0.00071,0.00047]$ \\
Direct global slice & $0.00000\;[0.00000,0.00000]$ & $0.00000\;[0.00000,0.00000]$ & $0.00000\;[0.00000,0.00000]$ \\
Window normalization & $0.00175\;[-0.00015,0.00380]$ & $-0.00046\;[-0.00170,0.00082]$ & $-0.00158\;[-0.00564,0.00126]$ \\
Aligned re-encoding & $0.00000\;[0.00000,0.00000]$ & $0.00000\;[0.00000,0.00000]$ & $0.00000\;[0.00000,0.00000]$ \\
\bottomrule
\end{tabular}
\end{adjustbox}
% ===== End inlined: tables/paired_geometry.tex =====
}
\end{table}

\begin{table}[htbp]
\centering
\caption{Selected paired effects after excluding overlap with the configured local training set ($n=22$). PSNR is continuation fidelity; LPIPS and temporal error use fixed seams.}
\label{tab:paired_sensitivity}
{\scriptsize\setlength{\tabcolsep}{1.7pt}%
% ===== Begin inlined: tables/paired_sensitivity.tex =====
% Generated by scripts/update_experiment_assets.py; do not edit numerical rows.
\begin{adjustbox}{max width=\linewidth}
\begin{tabular}{@{}lrccc@{}}
\toprule
Arm minus full & $n$ & $\Delta$PSNR$\uparrow$ & $\Delta$LPIPS$\downarrow$ & $\Delta E_{\rm temp}\downarrow$ \\
\midrule
Boundary RGB only & 22 & $0.00462\;[-0.01798,0.02533]$ & $0.00169\;[-0.00001,0.00331]$ & $-0.00051\;[-0.00152,0.00044]$ \\
Fresh history noise & 22 & $-0.01197\;[-0.08981,0.06875]$ & $-0.00170\;[-0.00328,-0.00012]$ & $0.00018\;[-0.00061,0.00107]$ \\
Window normalization & 22 & $-0.01149\;[-0.05585,0.03268]$ & $0.00075\;[-0.00075,0.00229]$ & $0.00086\;[-0.00022,0.00224]$ \\
\bottomrule
\end{tabular}
\end{adjustbox}
% ===== End inlined: tables/paired_sensitivity.tex =====
}
\end{table}

The seam comparison measures history participation near fixed output boundaries, while the continuation table measures its effect over the full suffix. The boundary-only seam LPIPS difference is 0.00195 on all 24 sources and 0.00169 on the 22-source sensitivity subset; the latter interval includes zero. Its continuation temporal-error difference is $-0.000324$ ($[-0.000449,-0.000195]$). Reporting both temporal scopes separates local seam behavior from the complete generated interval. State preservation is established by the exact checks independently of these quality measurements.

\subsection{State Consistency Checks}
\label{app:checks}
All 24 full-method replays reproduce their primary outputs exactly. Full and native-depth windows execute eight model evaluations; each alternative arm reuses its source's first-window output and executes eight evaluations in each continuation window. Constrained history rows have zero maximum absolute error. Local boundary extraction matches full decoding at window lengths of 92, 97, and 102 latents. Rolling and direct depth access also produce bitwise-equal video and audio latents in two 60-second cases. The quality evaluation concerns the visual outputs.

\begin{table}[htbp]
\centering
\caption{Temporal state consistency in completed model runs. Equality is checked on latent tensors or complete video outputs.}
\label{tab:checks}
\begin{tabularx}{\linewidth}{@{}lX@{}}
\toprule
Property & Result \\
\midrule
Reconstruct canonical generation & 24/24 full-method outputs exactly match \\
Incremental geometric reads & 24/24 rolling/direct outputs bitwise equal \\
Canonical encoding phase & 24/24 aligned re-encoded outputs bitwise equal \\
On-demand source equivalence & 3/3 online/cached outputs bitwise equal \\
Committed control-branch history & Initial-window latents bitwise equal \\
\bottomrule
\end{tabularx}
\end{table}

%
% ===== Begin inlined: appendix_vbench.tex =====
\subsection{VBench Quality Comparison: Protocol and Paired Results}
\label{app:vbench}
Table~\ref{tab:vbench} evaluates 800 existing videos from four configurations, two nominal durations, and the same 100 sources (25 per category; 96 source groups). The H3/VDN DMD8 row is the shared eight-evaluation generator used by \method{}. Its 5/14-second outputs occupy one window, characterizing within-window generation; Section~\ref{sec:ablations} evaluates the separate continuation mechanisms. The comparison retains each configuration's original weights, sampling, spatial canvas, and temporal rate.

\begin{table}[htbp]
\centering
\caption{Generation configurations for VBench.}
\label{tab:vbench_protocol}
{\small\setlength{\tabcolsep}{3.5pt}
\begin{adjustbox}{max width=\linewidth}
\begin{tabular}{@{}lrrrl@{}}
\toprule
Configuration & Canvas & FPS & Frames (5/14 s) & Sampling \\
\midrule
H3 50-step & $960\times544$ & 24 & 120/336 & 50 steps, CFG 1 \\
H3/VDN DMD8 & $960\times544$ & 24 & 120/336 & 8 NFE, CFG 1 \\
LongVie2 & $640\times352$ & 16 & 81/225 & 50 steps, CFG 5 \\
LongVie AR2 & $640\times352$ & 16 & 81/225 & 2 steps per latent, CFG 1 \\
\bottomrule
\end{tabular}
\end{adjustbox}}
\end{table}

All configurations use reference RGB and depth; LongVie additionally receives SpaTracker trajectories. The H3 pair shares input caches and seed policy, with different generator weights and attention implementations. The DMD8 checkpoint has zero local joint-training updates. H3 uses 124/345 internal frames before trimming. LongVie2 uses 81-frame windows with one-frame overlap, and its 5-second output is the first window of the 14-second run. LongVie AR2 uses a finite 12-latent window with eight sink latents. Preserving the LongVie endpoint gives encoded durations of 5.0625/14.0625 seconds. The results describe these complete configurations; their native resolutions, frame rates, and control inputs are part of the comparison.

\paragraph{Official evaluation and aggregation.}
We use VBench 0.1.5~\citep{huang2023vbench} with its official custom-video evaluation pipeline. Each metric reads the original video through its official loader. DINO and CLIP use the official frame transforms; MUSIQ uses the longer-side resizing policy; RAFT samples at approximately 8 fps; AMT-S uses original adjacent-frame intervals. Pixel-change flickering scores on these moving scenes are interpreted jointly with dynamic degree, which records detected movement.

For dimension score $s_d$, official bounds $m_d,M_d$, and weight $w_d$, the reported Quality Score is
\begin{equation}
 Q=\frac{1}{6.5}\sum_{d=1}^{7}w_d\frac{s_d-m_d}{M_d-m_d},
 \qquad w_{\mathrm{dynamic}}=0.5,\quad w_d=1\ \text{otherwise}.
 \label{eq:vbench_quality}
\end{equation}
Normalized values are retained without clipping. The bounds for subject consistency, background consistency, flickering, and smoothness are respectively $[0.1462,1]$, $[0.2615,1]$, $[0.6293,1]$, and $[0.706,0.9975]$; the other dimensions use $[0,1]$. Raw MUSIQ values are divided by 100 once. Equal-video dimension means are aggregated with equation~\ref{eq:vbench_quality}, and all table values are multiplied by 100. This is the quality component on our custom source collection; source fidelity and depth agreement are evaluated separately in Tables~\ref{tab:quality} and~\ref{tab:native}.

For the combined results in Table~\ref{tab:vbench}, we average the two duration scores within each source, then average over the 100 sources with equal weight. With balanced counts, this equals the mean over all 200 videos per configuration and gives each duration half the total weight. Table~\ref{tab:vbench_by_duration} retains the original duration-level means.

\begin{table}[htbp]
\centering
\caption{VBench quality by duration ($\times100$, $\uparrow$); best in bold.}
\label{tab:vbench_by_duration}
{\scriptsize\setlength{\tabcolsep}{3pt}%
% ===== Begin inlined: tables/vbench_by_duration.tex =====
% Generated by scripts/update_vbench_tables.py; do not edit numerical rows.
\begin{adjustbox}{max width=\linewidth}
\begin{tabular}{@{}lrrrrrrrr@{}}
\toprule
Configuration & \shortstack{Quality\\Score} & \shortstack{Subject\\cons.} & \shortstack{Background\\cons.} & \shortstack{Temporal\\flickering} & \shortstack{Motion\\smoothness} & \shortstack{Dynamic\\degree} & \shortstack{Aesthetic\\quality} & \shortstack{Imaging\\quality} \\
\midrule
\multicolumn{9}{@{}l}{\textit{5 s; $n=100$ per configuration}} \\
H3 50-step & 80.40 & \textbf{89.38} & 93.78 & \textbf{96.92} & 98.60 & 84.00 & 52.77 & 60.93 \\
LongVie2 & 79.12 & 88.64 & 93.59 & 96.45 & 97.83 & 86.00 & 50.57 & 58.83 \\
LongVie AR2 & 79.04 & 87.37 & 92.65 & 96.64 & 97.84 & \textbf{88.00} & 50.20 & 59.92 \\
H3/VDN DMD8 & \textbf{80.89} & 88.93 & \textbf{93.82} & 96.74 & \textbf{98.75} & 86.00 & \textbf{54.06} & \textbf{62.31} \\
\midrule
\multicolumn{9}{@{}l}{\textit{14 s; $n=100$ per configuration}} \\
H3 50-step & 79.34 & \textbf{85.55} & 92.10 & 96.53 & 98.34 & 83.00 & 52.42 & 63.64 \\
LongVie2 & 77.96 & 85.17 & 92.13 & 96.47 & 97.82 & 84.00 & 50.43 & 58.45 \\
LongVie AR2 & 78.19 & 85.19 & 91.73 & \textbf{96.62} & 97.84 & 84.00 & 49.92 & 60.56 \\
H3/VDN DMD8 & \textbf{80.18} & 85.48 & \textbf{92.37} & 96.35 & \textbf{98.60} & \textbf{85.00} & \textbf{55.08} & \textbf{64.71} \\
\bottomrule
\end{tabular}
\end{adjustbox}
% ===== End inlined: tables/vbench_by_duration.tex =====
}
\end{table}

\paragraph{Paired uncertainty.}
Tables~\ref{tab:vbench_paired_consistency} and~\ref{tab:vbench_paired_visual} report DMD8 minus each comparator on the same sources. We resample source groups within each category, retain fixed category weights and all model/metric pairings, and evaluate each duration separately. There are 10,000 bootstrap replicates with seed 20260925. Brackets are marginal 95\% percentile intervals without multiplicity adjustment. One fixed output per source/configuration contributes to each duration. All six Quality Score intervals are positive; Tables~\ref{tab:vbench_categories_real} and~\ref{tab:vbench_categories_rendered} retain the full category-level means.

\begin{table}[htbp]
\centering
\caption{Paired differences (DMD8 $-$ comparator, $\times100$); brackets: 95\% intervals.}
\label{tab:vbench_paired_consistency}
{\small\setlength{\tabcolsep}{4pt}%
% ===== Begin inlined: tables/vbench_paired_consistency.tex =====
% Generated by scripts/update_vbench_tables.py; do not edit numerical rows.
\begin{adjustbox}{max width=\linewidth}
\begin{tabular}{@{}lrrrr@{}}
\toprule
Configuration & \shortstack{Quality\\Score} & \shortstack{Subject\\cons.} & \shortstack{Background\\cons.} & \shortstack{Temporal\\flickering} \\
\midrule
\multicolumn{5}{@{}l}{\textit{5 s: DMD8 minus comparator; $n=100$ pairs}} \\
H3 50-step & \shortstack{$+0.49$\\$[+0.15,+0.85]$} & \shortstack{$-0.45$\\$[-0.81,-0.12]$} & \shortstack{$+0.04$\\$[-0.16,+0.25]$} & \shortstack{$-0.19$\\$[-0.29,-0.08]$} \\
LongVie2 & \shortstack{$+1.78$\\$[+1.39,+2.16]$} & \shortstack{$+0.29$\\$[-0.25,+0.78]$} & \shortstack{$+0.23$\\$[-0.02,+0.48]$} & \shortstack{$+0.29$\\$[+0.11,+0.45]$} \\
LongVie AR2 & \shortstack{$+1.85$\\$[+1.42,+2.26]$} & \shortstack{$+1.55$\\$[+0.96,+2.09]$} & \shortstack{$+1.17$\\$[+0.84,+1.50]$} & \shortstack{$+0.10$\\$[-0.05,+0.25]$} \\
\midrule
\multicolumn{5}{@{}l}{\textit{14 s: DMD8 minus comparator; $n=100$ pairs}} \\
H3 50-step & \shortstack{$+0.84$\\$[+0.48,+1.21]$} & \shortstack{$-0.07$\\$[-0.74,+0.57]$} & \shortstack{$+0.28$\\$[-0.19,+0.76]$} & \shortstack{$-0.18$\\$[-0.27,-0.09]$} \\
LongVie2 & \shortstack{$+2.22$\\$[+1.84,+2.61]$} & \shortstack{$+0.31$\\$[-0.40,+1.01]$} & \shortstack{$+0.24$\\$[-0.17,+0.65]$} & \shortstack{$-0.13$\\$[-0.36,+0.08]$} \\
LongVie AR2 & \shortstack{$+1.98$\\$[+1.55,+2.44]$} & \shortstack{$+0.29$\\$[-0.39,+0.95]$} & \shortstack{$+0.64$\\$[+0.23,+1.05]$} & \shortstack{$-0.27$\\$[-0.50,-0.05]$} \\
\bottomrule
\end{tabular}
\end{adjustbox}
% ===== End inlined: tables/vbench_paired_consistency.tex =====
}
\end{table}

\begin{table}[htbp]
\centering
\caption{Remaining paired VBench dimensions; notation as in Table~\ref{tab:vbench_paired_consistency}.}
\label{tab:vbench_paired_visual}
{\small\setlength{\tabcolsep}{4pt}%
% ===== Begin inlined: tables/vbench_paired_visual.tex =====
% Generated by scripts/update_vbench_tables.py; do not edit numerical rows.
\begin{adjustbox}{max width=\linewidth}
\begin{tabular}{@{}lrrrr@{}}
\toprule
Configuration & \shortstack{Motion\\smoothness} & \shortstack{Dynamic\\degree} & \shortstack{Aesthetic\\quality} & \shortstack{Imaging\\quality} \\
\midrule
\multicolumn{5}{@{}l}{\textit{5 s: DMD8 minus comparator; $n=100$ pairs}} \\
H3 50-step & \shortstack{$+0.15$\\$[+0.09,+0.22]$} & \shortstack{$+2.00$\\$[-1.45,+6.00]$} & \shortstack{$+1.29$\\$[+0.81,+1.77]$} & \shortstack{$+1.38$\\$[+0.62,+2.26]$} \\
LongVie2 & \shortstack{$+0.92$\\$[+0.72,+1.14]$} & \shortstack{$+0.00$\\$[-3.00,+3.00]$} & \shortstack{$+3.49$\\$[+2.68,+4.28]$} & \shortstack{$+3.48$\\$[+2.05,+4.95]$} \\
LongVie AR2 & \shortstack{$+0.91$\\$[+0.73,+1.12]$} & \shortstack{$-2.00$\\$[-6.00,+2.00]$} & \shortstack{$+3.86$\\$[+3.11,+4.59]$} & \shortstack{$+2.39$\\$[+1.30,+3.51]$} \\
\midrule
\multicolumn{5}{@{}l}{\textit{14 s: DMD8 minus comparator; $n=100$ pairs}} \\
H3 50-step & \shortstack{$+0.26$\\$[+0.18,+0.35]$} & \shortstack{$+2.00$\\$[+0.00,+5.00]$} & \shortstack{$+2.67$\\$[+1.84,+3.50]$} & \shortstack{$+1.07$\\$[+0.29,+1.89]$} \\
LongVie2 & \shortstack{$+0.78$\\$[+0.61,+0.96]$} & \shortstack{$+1.00$\\$[+0.00,+3.00]$} & \shortstack{$+4.65$\\$[+3.77,+5.49]$} & \shortstack{$+6.26$\\$[+4.85,+7.79]$} \\
LongVie AR2 & \shortstack{$+0.76$\\$[+0.60,+0.94]$} & \shortstack{$+1.00$\\$[-2.00,+4.00]$} & \shortstack{$+5.16$\\$[+4.20,+6.09]$} & \shortstack{$+4.15$\\$[+2.89,+5.46]$} \\
\bottomrule
\end{tabular}
\end{adjustbox}
% ===== End inlined: tables/vbench_paired_visual.tex =====
}
\end{table}

\begin{table}[htbp]
\centering
\caption{VBench quality for driving and human activity ($\times100$, $\uparrow$); best in bold.}
\label{tab:vbench_categories_real}
{\scriptsize\setlength{\tabcolsep}{3pt}%
% ===== Begin inlined: tables/vbench_categories_real.tex =====
% Generated by scripts/update_vbench_tables.py; do not edit numerical rows.
\begin{adjustbox}{max width=\linewidth}
\begin{tabular}{@{}lrrrrrrrr@{}}
\toprule
Configuration & \shortstack{Quality\\Score} & \shortstack{Subject\\cons.} & \shortstack{Background\\cons.} & \shortstack{Temporal\\flickering} & \shortstack{Motion\\smoothness} & \shortstack{Dynamic\\degree} & \shortstack{Aesthetic\\quality} & \shortstack{Imaging\\quality} \\
\midrule
\multicolumn{9}{@{}l}{\textit{Driving: 5 s; $n=25$ per configuration}} \\
H3 50-step & 79.34 & \textbf{93.94} & 96.14 & \textbf{98.39} & 99.12 & 76.00 & 44.42 & 52.12 \\
LongVie2 & 79.29 & 93.20 & 96.02 & 97.72 & 98.81 & \textbf{84.00} & 43.58 & 52.53 \\
LongVie AR2 & 79.02 & 90.72 & 94.62 & 97.74 & 98.87 & \textbf{84.00} & 45.05 & 53.87 \\
H3/VDN DMD8 & \textbf{80.31} & 93.91 & \textbf{96.37} & 98.20 & \textbf{99.16} & 80.00 & \textbf{46.36} & \textbf{54.56} \\
\midrule
\multicolumn{9}{@{}l}{\textit{Driving: 14 s; $n=25$ per configuration}} \\
H3 50-step & 78.93 & \textbf{90.71} & 94.76 & \textbf{98.21} & 99.06 & 80.00 & 44.00 & 54.25 \\
LongVie2 & 78.43 & 90.61 & 94.66 & 97.59 & 98.71 & 84.00 & 44.42 & 51.69 \\
LongVie AR2 & 78.12 & 88.72 & 93.41 & 97.67 & 98.80 & 80.00 & 44.43 & 55.02 \\
H3/VDN DMD8 & \textbf{80.22} & 90.26 & \textbf{94.92} & 97.99 & \textbf{99.09} & \textbf{88.00} & \textbf{47.06} & \textbf{56.32} \\
\midrule
\multicolumn{9}{@{}l}{\textit{Human activity: 5 s; $n=25$ per configuration}} \\
H3 50-step & 81.11 & \textbf{93.34} & 94.47 & 98.52 & 99.25 & 60.00 & 54.36 & 63.83 \\
LongVie2 & 80.21 & 93.26 & \textbf{94.72} & 98.03 & 98.82 & 64.00 & 52.67 & 60.24 \\
LongVie AR2 & 80.50 & 92.54 & 94.12 & 97.96 & 98.78 & \textbf{68.00} & 52.34 & 62.41 \\
H3/VDN DMD8 & \textbf{81.95} & 92.86 & 94.43 & \textbf{98.53} & \textbf{99.29} & \textbf{68.00} & \textbf{55.18} & \textbf{64.95} \\
\midrule
\multicolumn{9}{@{}l}{\textit{Human activity: 14 s; $n=25$ per configuration}} \\
H3 50-step & 80.45 & 91.00 & 92.77 & 98.23 & 99.07 & 60.00 & 54.70 & 65.71 \\
LongVie2 & 79.49 & \textbf{91.60} & \textbf{93.99} & 98.13 & 98.87 & 60.00 & 52.89 & 59.81 \\
LongVie AR2 & 79.79 & 91.56 & 93.53 & 97.96 & 98.76 & \textbf{64.00} & 51.69 & 62.51 \\
H3/VDN DMD8 & \textbf{80.96} & 90.70 & 93.31 & \textbf{98.24} & \textbf{99.20} & 60.00 & \textbf{56.34} & \textbf{66.50} \\
\bottomrule
\end{tabular}
\end{adjustbox}
% ===== End inlined: tables/vbench_categories_real.tex =====
}
\end{table}

\begin{table}[htbp]
\centering
\caption{Simulation and synthetic VBench quality; notation as in Table~\ref{tab:vbench_categories_real}.}
\label{tab:vbench_categories_rendered}
{\scriptsize\setlength{\tabcolsep}{3pt}%
% ===== Begin inlined: tables/vbench_categories_rendered.tex =====
% Generated by scripts/update_vbench_tables.py; do not edit numerical rows.
\begin{adjustbox}{max width=\linewidth}
\begin{tabular}{@{}lrrrrrrrr@{}}
\toprule
Configuration & \shortstack{Quality\\Score} & \shortstack{Subject\\cons.} & \shortstack{Background\\cons.} & \shortstack{Temporal\\flickering} & \shortstack{Motion\\smoothness} & \shortstack{Dynamic\\degree} & \shortstack{Aesthetic\\quality} & \shortstack{Imaging\\quality} \\
\midrule
\multicolumn{9}{@{}l}{\textit{Simulation: 5 s; $n=25$ per configuration}} \\
H3 50-step & 78.94 & \textbf{82.29} & \textbf{91.15} & 95.71 & 97.91 & \textbf{100.00} & 51.03 & 62.68 \\
LongVie2 & 76.72 & 81.07 & 90.00 & 95.23 & 96.33 & 96.00 & 48.53 & 62.46 \\
LongVie AR2 & 77.04 & 79.94 & 89.39 & \textbf{96.00} & 96.76 & \textbf{100.00} & 47.74 & 61.97 \\
H3/VDN DMD8 & \textbf{79.28} & 82.20 & 91.05 & 95.70 & \textbf{98.31} & 96.00 & \textbf{52.97} & \textbf{63.83} \\
\midrule
\multicolumn{9}{@{}l}{\textit{Simulation: 14 s; $n=25$ per configuration}} \\
H3 50-step & 77.39 & 77.65 & \textbf{89.25} & 95.14 & 97.53 & \textbf{92.00} & 51.76 & 66.71 \\
LongVie2 & 75.45 & 76.92 & 88.63 & 95.35 & 96.42 & \textbf{92.00} & 48.18 & 62.62 \\
LongVie AR2 & 76.08 & 77.65 & 88.67 & \textbf{95.95} & 96.81 & \textbf{92.00} & 47.80 & 63.25 \\
H3/VDN DMD8 & \textbf{78.15} & \textbf{77.68} & 88.94 & 95.03 & \textbf{98.07} & \textbf{92.00} & \textbf{54.21} & \textbf{68.07} \\
\midrule
\multicolumn{9}{@{}l}{\textit{Synthetic: 5 s; $n=25$ per configuration}} \\
H3 50-step & \textbf{82.20} & \textbf{87.93} & 93.34 & \textbf{95.07} & 98.11 & \textbf{100.00} & 61.28 & 65.08 \\
LongVie2 & 80.24 & 87.02 & \textbf{93.61} & 94.82 & 97.37 & \textbf{100.00} & 57.48 & 60.10 \\
LongVie AR2 & 79.61 & 86.30 & 92.46 & 94.86 & 96.97 & \textbf{100.00} & 55.69 & 61.42 \\
H3/VDN DMD8 & 82.03 & 86.75 & 93.42 & 94.52 & \textbf{98.24} & \textbf{100.00} & \textbf{61.73} & \textbf{65.90} \\
\midrule
\multicolumn{9}{@{}l}{\textit{Synthetic: 14 s; $n=25$ per configuration}} \\
H3 50-step & 80.59 & 82.85 & 91.61 & 94.52 & 97.69 & \textbf{100.00} & 59.21 & 67.88 \\
LongVie2 & 78.47 & 81.55 & 91.26 & 94.81 & 97.29 & \textbf{100.00} & 56.23 & 59.69 \\
LongVie AR2 & 78.78 & 82.85 & 91.30 & \textbf{94.89} & 96.98 & \textbf{100.00} & 55.78 & 61.47 \\
H3/VDN DMD8 & \textbf{81.38} & \textbf{83.28} & \textbf{92.32} & 94.12 & \textbf{98.05} & \textbf{100.00} & \textbf{62.72} & \textbf{67.94} \\
\bottomrule
\end{tabular}
\end{adjustbox}
% ===== End inlined: tables/vbench_categories_rendered.tex =====
}
\end{table}
% ===== End inlined: appendix_vbench.tex =====

\FloatBarrier
\section{Runtime, Resources, and Applications}
\label{app:runtime}

\subsection{Controlled DiT Timing}
Primary timings cover synchronized model forward passes, including the H3/VDN backbone, Union branch, and FSDP/Ulysses communication. Each rank synchronizes local CUDA before and after each forward. We sum the maximum rank time at each of eight steps per window. Packing, Euler updates, history projection, depth preparation, VAE operations, validation, and output are excluded. Four fixed sources per canvas receive a full-window warmup and three hot repeats on isolated eight-H100 workers. The 32 total requests include warmups; reported hot means use 12 requests per canvas.

\begin{table}[htbp]
\centering
\caption{DiT timing by window, means over 12 hot requests at each resolution. Full-continuation windows meet the 14-FPS throughput target at both resolutions. The last window delivers only 69 frames after overlap removal and padding trim. Step time is the mean of synchronized maximum-rank forward times.}
\label{tab:dit_windows}
{\small\setlength{\tabcolsep}{3.2pt}%
% ===== Begin inlined: tables/dit_windows.tex =====
% Generated by scripts/update_experiment_assets.py; do not edit numerical rows.
\begin{adjustbox}{max width=\linewidth}
\begin{tabular}{@{}llrrrrr@{}}
\toprule
Canvas & Window & New frames & DiT (s) & Step (s) & New FPS & RTF@24fps \\
\midrule
$960\times544$ & Initial & 345 & 21.144 & 2.643 & 16.346 & 1.471 \\
$960\times544$ & Full continuation & 306 & 21.175 & 2.647 & 14.472 & 1.661 \\
$960\times544$ & Cropped tail & 69 & 10.455 & 1.307 & 6.616 & 3.636 \\
$672\times384$ & Initial & 345 & 14.667 & 1.833 & 23.548 & 1.020 \\
$672\times384$ & Full continuation & 306 & 14.356 & 1.795 & 21.361 & 1.126 \\
$672\times384$ & Cropped tail & 69 & 8.571 & 1.071 & 8.072 & 2.981 \\
\bottomrule
\end{tabular}
\end{adjustbox}
% ===== End inlined: tables/dit_windows.tex =====
}
\end{table}

\begin{table}[htbp]
\centering
\caption{Whole-trajectory DiT throughput. Mean FPS averages each hot request's ratio. Pooled FPS divides all delivered frames by all elapsed DiT time; pooled steady FPS uses only full continuation windows.}
\label{tab:dit_whole}
{\small\setlength{\tabcolsep}{3pt}%
% ===== Begin inlined: tables/dit_whole.tex =====
% Generated by scripts/update_experiment_assets.py; do not edit numerical rows.
\begin{adjustbox}{max width=\linewidth}
\begin{tabular}{@{}lrrrrr@{}}
\toprule
Canvas & Total (s) & Mean FPS & RTF@24fps & Pooled FPS & Pooled steady FPS \\
\midrule
$960\times544$ & 52.774 & 13.665 & 1.759 & 13.643 & 14.451 \\
$672\times384$ & 37.594 & 19.183 & 1.253 & 19.152 & 21.315 \\
\bottomrule
\end{tabular}
\end{adjustbox}
% ===== End inlined: tables/dit_whole.tex =====
}
\end{table}

The 345/306/69 new-frame counts sum to the 720-frame delivery in the 24-FPS evaluation stream. A complete continuation has 12.75 seconds of evaluation playback and a 1.59375-second average forward budget at eight evaluations. The $672\times384$ and $960\times544$ full-continuation paths reach 21.361 and 14.472 new FPS, respectively, meeting the 14-FPS throughput target at both resolutions. Timing runs reproduce the final latents of the corresponding requests.

\subsection{Measured Depth-State and Model Resources}
\label{app:resources}
Resource probes use one warmup and one measured full first window per canvas, with eight actual forwards on each of eight ranks. Depth rows, carrier, and Union hints are measured from tensors. Local parameter storage and allocator peaks are measured separately. Incremental peaks include temporary activations and FSDP gather buffers; these components cannot be additively separated from this probe. The video VAE is not resident in the probe, so these measurements describe the measured model and conditioning state rather than a full-renderer peak.

\begin{table}[htbp]
\centering
\caption{Actual first-window tensors and allocator increments. Hints are per rank; model storage and allocation increase take the maximum rank. Units distinguish MiB tensors from GiB model/allocator quantities.}
\label{tab:measured_resources}
{\small\setlength{\tabcolsep}{3.5pt}%
% ===== Begin inlined: tables/resources.tex =====
% Generated by scripts/update_experiment_assets.py; do not edit numerical rows.
\begin{adjustbox}{max width=\linewidth}
\begin{tabular}{@{}lrrrrr@{}}
\toprule
Canvas & Depth (MiB) & Carrier (MiB) & Hints (MiB) & Model (GiB) & Increase (GiB) \\
\midrule
$960\times544$ & 19.050 & 39.276 & 347.710 & 9.025 & 7.034 \\
$672\times384$ & 9.413 & 19.407 & 177.393 & 9.025 & 5.090 \\
\bottomrule
\end{tabular}
\end{adjustbox}
% ===== End inlined: tables/resources.tex =====
}
\end{table}

\begin{figure}[t]
\centering
\includegraphics[width=\linewidth]{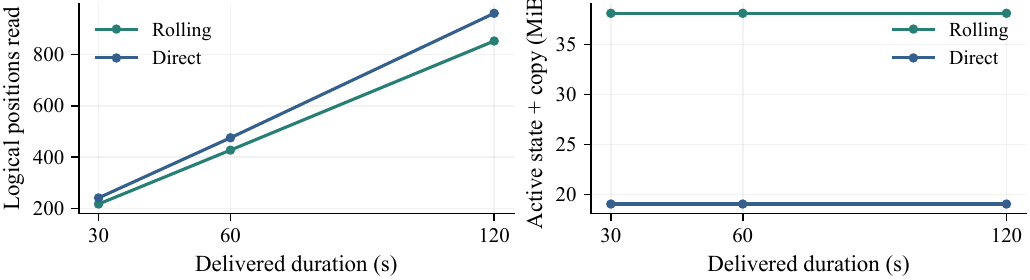}
\caption{Logical depth reads and retained state plus consumer copy at $960\times544$. The resident global source is additional and grows with duration. Rolling reduces read positions but retains an extra owned state tensor.}
\label{fig:depth_resources}
\end{figure}

The CPU benchmark uses actual FP32 depth, two PyTorch threads, one warmup, and 100 measurements per case, with the same validation/copy semantics for both readers. Four 30-second sources at each canvas and two longer geometries give ten cases. Table~\ref{tab:state_runtime} averages per-source switch means and, separately, per-source p95 values; the latter is not a pooled p95 and differs from the DiT-window statistic.

\begin{table}[htbp]
\centering
\caption{Resident CPU source switching, four 30-second sources per canvas, 100 repeats per source. State + copy excludes the global source.}
\label{tab:state_runtime}
{\small\setlength{\tabcolsep}{3.5pt}%
% ===== Begin inlined: tables/state_runtime.tex =====
% Generated by scripts/update_experiment_assets.py; do not edit numerical rows.
\begin{adjustbox}{max width=\linewidth}
\begin{tabular}{@{}llrrr@{}}
\toprule
Canvas & Reader & Mean switch (ms) & Mean source p95 (ms) & State + copy (MiB) \\
\midrule
$960\times544$ & Rolling & 11.836 & 18.088 & 38.101 \\
$960\times544$ & Direct & 7.377 & 11.509 & 19.050 \\
$672\times384$ & Rolling & 3.502 & 6.026 & 18.826 \\
$672\times384$ & Direct & 2.807 & 4.253 & 9.413 \\
\bottomrule
\end{tabular}
\end{adjustbox}
% ===== End inlined: tables/state_runtime.tex =====
}
\end{table}

At $960\times544$, rolling requests 217/427/852 unique positions for 30/60/120-second deliveries; direct reads request 241/475/960. These counts measure logical source requests. Rolling concatenates retained overlap, while direct slicing accesses the resident source; Table~\ref{tab:state_runtime} records the corresponding switch costs. Global depth, global noise, and committed outputs scale with duration and are accounted for separately. For 30 seconds at $960\times544$, global depth and noise occupy approximately 40.529 and 41.015 MiB, additional to active state and the carrier.

\subsection{Long-Horizon Results}
\begin{table}[htbp]
\centering
\caption{Long-video aggregates. The 17 sixty-second sources use existing validation/test splits; the thirteen 120-second sources use an existing training split and are stress tests. Source sets differ.}
\label{tab:long_quality}
{\small\setlength{\tabcolsep}{3.8pt}%
% ===== Begin inlined: tables/long_quality.tex =====
% Generated by scripts/update_experiment_assets.py; do not edit numerical rows.
\begin{adjustbox}{max width=\linewidth}
\begin{tabular}{@{}rrrrrrrr@{}}
\toprule
Seconds & $n$ & PSNR$\uparrow$ & SSIM$\uparrow$ & LPIPS$\downarrow$ & Rank$\uparrow$ & Edge$\uparrow$ & NMAE$\downarrow$ \\
\midrule
60 & 17 & 14.4542 & 0.3198 & 0.4903 & 0.9114 & 0.6928 & 0.0999 \\
120 & 13 & 14.4277 & 0.3358 & 0.5464 & 0.9353 & 0.6809 & 0.1771 \\
\bottomrule
\end{tabular}
\end{adjustbox}
% ===== End inlined: tables/long_quality.tex =====
}
\end{table}

The 60-second PSNR confidence interval is $[13.5573,15.3446]$ dB; its LPIPS interval is $[0.43966,0.54186]$. For 120 seconds these intervals are $[13.4260,15.4057]$ dB and $[0.50169,0.59344]$. Fixed-scale NMAE is 0.09993 ($[0.06854,0.13587]$) and 0.17714 ($[0.05376,0.38348]$), respectively.

Over the first to last ten-second block, 60-second PSNR changes from 18.257 to 13.136 dB, and fixed-scale NMAE from 0.0580 to 0.1428. The 120-second values change from 18.665 to 13.307 dB and from 0.0903 to 0.2248. Figure~\ref{fig:long_horizon} characterizes these time-dependent measurements with source-group confidence bands while keeping geometric calibration fixed throughout each trajectory.

\subsection{On-Demand Geometry, Appearance, and Control}
\label{app:online}
Three 30-second, $672\times384$ procedural trajectories share one small PyBullet room. Warm/cool conditions use the same camera path and geometry with different reference materials; their depth is bitwise identical. The camera returns to its initial viewpoint at frame 480 (20 seconds). The control branch is enabled only when requesting the second window, with changes beginning at frame 357; the initial request has already read geometry through frame 356.

The cached cases run before the online cases; the first cached warm case may include cold compilation. Online execution requests new engine depth, encodes canonical blocks, and delivers RGB intervals $[0,340)$, $[340,646)$, and $[646,720)$. Every newly acquired depth interval matches the cached oracle, and all three final video-latent outputs match bitwise. The recorded memory includes the retained global oracle, noise, and output latents.

\begin{table}[htbp]
\centering
\caption{Appearance, native inverse-depth agreement, and revisit error for the three online conditions. Revisit columns compare frame 480 with frame zero using LPIPS; smaller is better. These are individual cases, not replicated population estimates.}
\label{tab:applications}
{\small\setlength{\tabcolsep}{3.3pt}%
% ===== Begin inlined: tables/applications.tex =====
% Generated by scripts/update_experiment_assets.py; do not edit numerical rows.
\begin{adjustbox}{max width=\linewidth}
\begin{tabular}{@{}lrrrrrr@{}}
\toprule
Condition & LPIPS$\downarrow$ & Rank$\uparrow$ & Edge$\uparrow$ & NMAE$\downarrow$ & Revisit pred. & Revisit GT \\
\midrule
Warm & 0.4558 & 0.8915 & 0.5280 & 0.1170 & 0.4769 & 0.0056 \\
Cool & 0.2473 & 0.9115 & 0.5686 & 0.0986 & 0.2635 & 0.0064 \\
Control & 0.4561 & 0.8909 & 0.5253 & 0.1170 & 0.4668 & 0.0118 \\
\bottomrule
\end{tabular}
\end{adjustbox}
% ===== End inlined: tables/applications.tex =====
}
\end{table}

Warm/cool generated RGB differs by a mean 0.19322 on $[0,1]$, versus 0.14456 for engine images; mean palette-change error is 0.03171. The control pair's initial-window latents are identical, while suffix latent MAE is 0.039076. The shared generator therefore responds to appearance and geometric-branch changes while preserving committed history. Table~\ref{tab:applications} reports absolute appearance, geometry, and revisit measurements for each case.

\begin{figure}[!t]
\centering
\includegraphics[width=\linewidth]{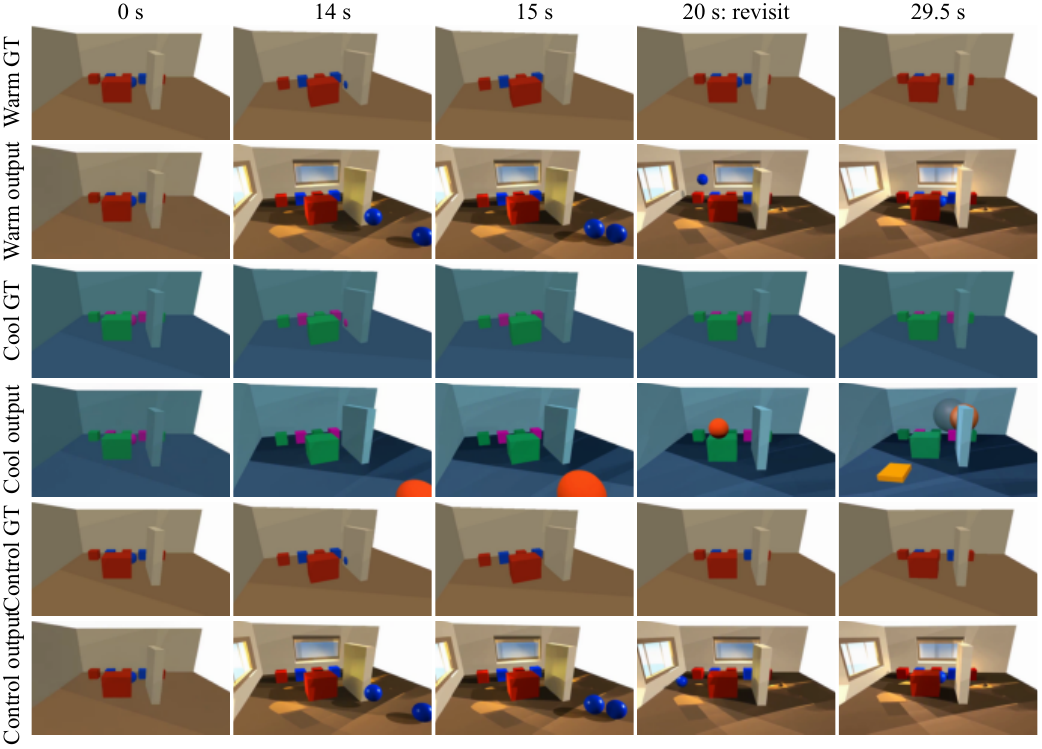}
\caption{\textbf{One shared renderer for appearance and geometric control.} Engine RGB and generated frames for warm/cool references and a control branch, at fixed times. Warm/cool conditions share depth; control changes begin at frame 357 and the camera revisits at 20 seconds.}
\label{fig:online}
\end{figure}

\subsection{Window and Overlap Configurations}
\label{app:overlap_results}
Four 30-second scenes---architecture, coast, gym, and office---compare overlaps 39/22/5 at $672\times384$. The corresponding windows are 345/328/311 frames with 12/7/2 history latents, all with stride 306. First-window content and commitment points change, unlike the fixed-window comparisons above.

\begin{table}[htbp]
\centering
\caption{Joint window-and-overlap configurations. PSNR averages per-frame dB; SSIM uses an $11\times11$ Gaussian window with $\sigma=1.5$. These definitions differ from the primary evaluation.}
\label{tab:overlap}
\begin{adjustbox}{max width=\linewidth}
\begin{tabular}{@{}rrrrrr@{}}
\toprule
Overlap & Window & History & Full PSNR & Full SSIM & Common-cont. PSNR \\
\midrule
39 & 345 & 12 & 13.842738 & 0.341540 & 12.788645 \\
22 & 328 & 7 & 13.683256 & 0.332978 & 12.861136 \\
5 & 311 & 2 & 13.685815 & 0.336063 & 12.866656 \\
\bottomrule
\end{tabular}
\end{adjustbox}
\end{table}

The full-video ranking favors 39-frame overlap; the shared continuation $[345,720)$ slightly favors shorter overlaps. A separate two-scene 60-second comparison also showed no consistent PSNR/SSIM advantage from history. These are configuration observations, not isolated estimates of the history constraint.

\subsection{Qualitative Selection and Paired Visualizations}
\label{app:additional}
Figure~\ref{fig:examples} uses the first selected ablation source from each category, independently of quality scores. The same four sources appear in the paired seam figures. Each row shows the same absolute timestamps across methods, and the reference is the actual conditioning image.

The figures pair source RGB with identical-time generated observations, complementing the source-fidelity, geometric, and temporal metrics. The simulation and synthetic rows display layout and appearance across both continuation boundaries. The online figure further varies reference appearance and future geometry within one shared model.
% ===== End inlined: appendix_experiments.tex =====

% ===== End inlined: appendices.tex =====

%
% ===== Begin inlined: appendix_statements.tex =====
\FloatBarrier
\section{AI Use and Reproducibility Statements}
\label{app:statements}

\subsection{AI use statement}
Generative AI assisted manuscript organization, drafting, editing, equation and pseudocode formatting, and preparation of tables and figures. Quantitative results and qualitative examples derive from completed model runs; method diagrams are illustrations. The authors remain responsible for the experimental results and final manuscript.

\subsection{Reproducibility statement}
Equations~\ref{eq:projection}--\ref{eq:fixed_depth} and Algorithm~\ref{alg:uniworld} specify inference. Appendices~\ref{app:proofs}--\ref{app:evaluation} provide state proofs, model components, VAE phase, depth conventions, valid masks, source-selection rules, paired controls, and bootstrap statistics. Appendix~\ref{app:runtime} specifies the hardware, warmup and repetition settings, and measurement scopes for DiT and resource measurements; the online subsection reports source-interface behavior and control equivalence. The supplementary tables report aggregate results and uncertainty under these protocols.
% ===== End inlined: appendix_statements.tex =====